\documentclass{article}
\usepackage[letterpaper, total={7in, 9.5in}]{geometry}

\usepackage[svipsnames,svgnames,x11names,table]{xcolor}
\usepackage{authblk}
\usepackage{url}
\usepackage[square,numbers,sort]{natbib}
\usepackage{indentfirst}

\newcommand\blfootnote[1]{%
  \begingroup%
  \renewcommand\thefootnote{}\footnote{#1}%
  \addtocounter{footnote}{-1}%
  \endgroup
}

\usepackage{settings/preamble}

\newcommand{\MDP}{\mathfrak{M}}
\newcommand{\M}{\mathcal{M}}

\renewcommand{\S}{\mathcal{S}}

\newcommand{\A}{\mathcal{A}}
\renewcommand{\P}{\mathbb{P}}

\newcommand{\bphi}{\bm{\phi}}

\newcommand{\bmu}{\bm{\mu}}
\newcommand{\bnu}{\bm{\nu}}
\newcommand{\bomega}{\bm{\omega}}

\newcommand{\bxi}{\bm{\xi}}
\newcommand{\bdelta}{\bm{\delta}}

\newcommand{\Vnom}{V}
\newcommand{\Vrob}{\widetilde{V}}
\newcommand{\Vrlx}{\widehat{V}}

\newcommand{\Qnom}{Q}
\newcommand{\Qrob}{\widetilde{Q}}
\newcommand{\Qrlx}{\widehat{Q}}

\newcommand{\Bnom}{\mathcal{B}}
\newcommand{\Brob}{\widetilde{\mathcal{B}}}
\newcommand{\Brlx}{\widehat{\mathcal{B}}}

\renewcommand{\r}{R}

\newcommand{\pirob}{\widetilde{\pi}}
\newcommand{\pirlx}{\widehat{\pi}}

\renewcommand{\o}{\mathrm{o}}

\makeatletter
\let\raw@overline\overline
\renewcommand{\overline}[1]{\mathrlap{\raw@overline{\phantom{#1}}}#1}
\let\raw@underline\underline
\renewcommand{\underline}[1]{\mathrlap{\raw@underline{\phantom{#1\!}}}#1}
\makeatother

\newcommand{\papertitle}{Efficient Duple Perturbation Robustness in Low-rank MDPs}

\allowdisplaybreaks[4]

\begin{document}
  \setlength{\abovedisplayskip}{3pt}
  \setlength{\abovedisplayshortskip}{2pt}
  \setlength{\belowdisplayskip}{3pt}
  \setlength{\belowdisplayshortskip}{2pt}
  \setlength{\jot}{1pt}  
  \setlength{\floatsep}{1ex}
  \setlength{\textfloatsep}{1ex}

  \title{\bfseries \papertitle}

\author[1]{Yang Hu\thanks{yanghu@g.harvard.edu}}
\author[1]{Haitong Ma\thanks{haitongma@g.harvard.edu}}
\author[2,3]{Bo Dai\thanks{bodai@cc.gatech.edu}}
\author[1]{Na Li\thanks{nali@seas.harvard.edu}}
\affil[1]{Harvard University}
\affil[2]{Google DeepMind}
\affil[3]{Georgia Institute of Technology}

\date{}

\maketitle
\blfootnote{\itshape This paper is supported by NSF AI institute: 2112085, NSF CNS: 2003111, NSF ECCS: 2328241.}
\vspace{-30pt}

\begin{abstract}
  The pursuit of robustness has recently been a popular topic in reinforcement learning~(RL) research, yet the existing methods generally suffer from efficiency issues that obstruct their real-world implementation. 
  In this paper, we introduce \textit{dual perturbation} robustness, i.e. perturbation on both the feature and factor vectors for low-rank Markov decision processes~(MDPs), via a novel characterization of $(\xi,\eta)$-ambiguity sets.  
  The novel robust MDP formulation is compatible with the function representation view, and therefore, is naturally applicable to practical RL problems with large or even continuous state-action spaces. Meanwhile, it also gives rise to a provably efficient and practical algorithm with theoretical convergence rate guarantee. 
  Examples are designed to justify the new robustness concept, and algorithmic efficiency is supported by both theoretical bounds and numerical simulations.
\end{abstract}
  \section{Introduction}

The recent years have witnessed the rapid development of reinforcement learning (RL), a discipline that has shown its power in various areas, ranging from gaming \citep{mnih2015human, silver2017mastering, vinyals2019grandmaster}, robotics \citep{levine2016end, dalal2021accelerating}, 
to large language models \citep{ziegler2019fine, ouyang2022training}.

These successes of RL are achieved with huge amounts of data, yet for applications like robotics and autonomous driving, the real-world data collection is expensive and dangerous. As a result, RL agents are usually trained (and even tested) in simulated environments.
However, simulators are typically not accurate for revealing the intrinsic uncertainty and ambiguity in the dynamics of real-world environments. Therefore, RL agents trained in simulated environments generally suffer from a sim-to-real performance degradation \citep{peng2018sim, salvato2021crossing}, which is the cost for overlooking the distributional shift from simulation to real-world environments.
 
In order to mitigate the performance degradation, robust MDPs have been proposed, in which the value of a policy is evaluated with respect to the worst possible realization in some ambiguity set of the dynamics. 
The modern formulation of robust MDPs shall be attributed to \citet{iyengar2005robust, el2005robust}, with a whole series of theoretical efforts aiming at solving the robust MDP \citep{goyal2023robust, ho2018fast, ho2021partial, zhou2021finite, yang2022toward, blanchet2023double}. Nevertheless, most of these theoretical works consider \emph{tabular MDPs}, and the computational complexity and/or sample complexity are, at best, polynomial in the size of state and action spaces. Therefore, these algorithms are bound to suffer from the \textit{curse of dimensionality}, which makes these algorithms impossible to deploy in environments with large or even infinite state-action spaces, and thus significantly restricts its modelling ability of real-world environments.

Towards alleviation of computational burden, people have been working to identify structures that reduce the intrinsic dimensionality of the problem. Among these attempts, MDPs with low-rank structure~\cite{jin2019LSVI-UCB, yang2020reinforcement} turn out to be a representative model family with practical implications~\citep{ren2022free,ren2023stochastic}. Specifically, in low-rank MDPs we assume that the transition probability kernels can be represented by an inner product of the \textit{feature maps} $\bphi(s, a)\in \R^d$ and the \textit{factors} $\bmu(s')\in \R^d$ (i.e., $\mathbb{P}(s'|s, a) = \angl*{\bphi(s, a), \bmu(s')}$; similar for the reward functions).  
Along this line of work, provably efficient algorithms have been developed such that the sample complexity is polynomial in $d$, the dimension of the feature space~\citep{agarwal2020flambe,uehara2021representation}.

Despite these promising results for low-rank MDPs, this line of work in general only deals with the non-robust case, shedding very limited light on the robust RL side. There have also been some initial efforts to introduce the idea of function approximation into the realm of robust MDPs~\citep{tamar2014scaling, badrinath2021robust}, yet they largely ignore the connection between the representation of value functions and the ambiguity set of MDPs. 
\citet{ma2023distributionally,goyal2023robust} do exploit such structures via the construction of ambiguity sets through $\bmu(s')$, but they assume fixed feature $\bphi(s, a)$, and consequently, restricts the scope of their robustness concept. 

To the best of our knowledge, none of the existing results deal with the uncertainty of the feature maps $\bphi(s, a)$ and the factors $\bmu(s')$ together. However, in real-world systems features may be random \cite{ren2023stochastic}, kernel-induced \cite{ren2022free} or generated by latent variables \cite{ren2022latent}, and are thus subject to uncertainty as well. Therefore, it is crucial to consider \textit{dual perturbation} on both feature and factor vectors, which motivates the following research question: 
\begin{adjustwidth}{2em}{2em}
\begin{center}
  \itshape How to design a \textbf{provable} and \textbf{efficient} algorithm for generic robustness in MDPs with both feature $\bphi(s,a)$ and factor $\bmu(s')$ uncertainty?
\end{center}
\end{adjustwidth}
By ``provable" we mean the suboptimality and convergence rate can be rigorously characterized, and by ``efficient" we mean the algorithm is computationally efficient so that it can be practically implemented.  
In this paper, we provide an {\bf affirmative} answer to this question. More concretely,
\begin{itemize}
  \item We propose a novel robustness concept via \textit{$(\xi,\eta)$-rectangular} ambiguity sets that is compatible with low-rank MDPs, in which the optimal policy displays certain level of robust behavior. We also present a few key properties that relate it to existing robustness concepts.
  \item We design the novel R\tsup{2}PG algorithm to solve the proposed robust low-rank MDPs with $(\xi,\eta)$-rectangularity. The algorithm is tractable as the optimization involved can be reduced to an SDP, and is thus potentially scalable to work with large state-action spaces.
  \item We provide a convergence guarantee for our R\tsup{2}PG algorithm that ensures provably efficient convergence to the optimal policy with bounded suboptimality.
\end{itemize}

\subsection{Related Work}

\paragraph{Robust MDPs and robust RL.} The study of robust MDPs can be traced back to the 1960s \citep{silver1963markovian, satia1973markovian}, when people first realized the importance of planning with uncertain dynamics and rewards, yet failed to specify the construction of ambiguity sets. The modern formulation of robust MDPs shall be attributed to \citet{iyengar2005robust, el2005robust}, which formally define the structure of robust MDPs in terms of $(s,a)$-rectangular ambiguity sets, and give a comprehensive exposition of their properties. The scope of ambiguity sets is later generalized to generic parameterized ambiguity sets by \citet{wiesemann2013robust}, which also reveals the theoretical hardness of solving robust MDPs. Later, there is a full line of work focusing on improving the computational efficiency of robust planning with rectangular ambiguity sets \citep{ho2018fast, ho2021partial}, and yet another line focusing on the learning of robust policies using offline data \citep{zhou2021finite, yang2022toward, blanchet2023double}. Most of these results, however, are only applicable under standard rectangularity.

On the empirical side, with the rapid development of deep learning, online learning of robust policies have become a topic of interest. Different attempts to tackle this problem are proposed in the past decade \citep{rajeswaran2016epopt, pattanaik2017robust, pinto2017robust, zhang2020robust}, yet these methods are largely practice-oriented, for which theoretical understandings are very much limited.

\paragraph{MDPs with linear/low-rank representations.} With an aim of reducing the computational time needed for reinforcement learning, people have realized the importance of representation (a.k.a. function approximation). People tend to leverage the low-rank structures in MDPs to design algorithms that reduce the dependency on the size of state-action spaces, and thus lead to more scalable performance.

Such effort leads to linear MDPs, which assumes linearly representable transition probabilities and rewards through feature maps \citep{jin2019LSVI-UCB, yang2020reinforcement, ren2022spectral}. A similar line of work considers low-rank MDPs, which basically are linear MDPs with unknown features, so that the learning of features should also be handled online \citep{agarwal2020flambe, uehara2021representation}.

\paragraph{Robust RL with function approximation.} Despite the direct relevance of the area to this paper, there has been very limited progress in this direction.

An earlier line of work assumes linear function approximation for $V$-functions. \citet{tamar2014scaling} consider infinite-horizon robust MDPs and assume that value functions lie in a \textit{known} feature space, for which they use a projected value iteration to robustly evaluate any given policy, and further use policy iteration to update the policy towards optimum. However, the feature space itself is assumed to be free from uncertainty. A follow-up paper \citep{badrinath2021robust} settles the problem of online model-free robust RL with linear approximation of $V$ by a least-squares policy iteration algorithm. These early results generally ignore the connection between representation and dynamics, and hence requires strong assumptions on the features that are difficult to justify and evaluate in practice. 

In more recent papers, the linear representation of rewards and transition probabilities is considered for robust MDPs. \citet{goyal2023robust, ma2023distributionally} consider a special family of soft state-aggregate MDPs, for which they assume the transition probability and rewards to lie in a \emph{known} feature space, while the factors of transition probabilities lie in a KL-constrained ambiguity set. Ambiguity sets beyond standard rectangularity bearing certain kind of ``low-rankness'' are proposed alongside, which they call factorizable $d$-rectangularity. A modified least-squares value iteration algorithm is designed to solve the offline learning problem. Such ambiguity set definition introduces some level robustness, but still, the perturbation is only with respect to factors with the feature maps fixed. 

As far as we are concerned, our paper is the first to consider uncertainties in both the feature maps and the factors in linear/low-rank MDPs.
  \section{Preliminaries}

In this section, we introduce some basic notations and definitions in reinforcement learning and representation.

\paragraph{Notations.} Let $\norm{\cdot}$ denote the Euclidean 2-norm, $\angl{\cdot, \cdot}$ denote the standard inner product, and $\norm{\cdot}_{\infty}$ denote the $\infty$-norm. Let $a:b$ denote an index running from $a$ to $b$ (both sides inclusive). Write $[n]$ for the set $\set{1,2,\ldots,n}$. For any set $S$, let $\Delta(S)$ denote the probability simplex over $S$. For a matrix $M$, write $M \succeq 0$ if $M$ is positive semi-definite.

\paragraph{Markov Decision Processes (MDPs).} We consider a \textit{finite-horizon MDP}, which is described by a tuple $\MDP = (H, \S, \A, \brac{\P_h}, \brac{r_h}, \rho)$. Here $H$ is the \textit{horizon} or the length of each episode, $\S$ is the \textit{state space}, and $\A$ is the \textit{action space}; at step $h \in [H]$, $\P_h: \S \times \A \to \Delta(\S)$ is the \textit{transition probability kernel}, while $r_h: \S \times \A \to [0,1]$ is the \textit{reward function}; $\rho \in \Delta(\S)$ is the initial state distribution. A (potentially non-stationary) \textit{policy} $\pi = (\pi_1, \ldots, \pi_H)$ is composed of stage policies $\pi_h: \S \to \Delta(\A)$, which determines a distribution over actions for each observed state at time step $h \in [H]$.

For an MDP with transition probability kernel $\P$, given policy $\pi$, let $\mathbb{E}_{\pi, \P}$ denote the expectation over a trajectory with prescribed initial condition, which evolves by $a_{\tau} \sim \pi_{\tau}(\cdot | s_{\tau}), s_{\tau+1} \sim \P_{\tau}(\cdot | s_{\tau}, a_{\tau})$ (the domain of $\tau$ shall be inferred from context). Define $\mathrm{Pr}_{\pi,\P}$ in a similar way.

For any given policy $\pi$, the standard $V$- and $Q$-functions starting from step $h \in [H]$ are defined as
\begin{align}\label{eq:2-nominal_value}
  V_h^{\pi}(s) &:=\E[\pi,\P]{\textstyle\sum_{\tau=h}^{H} r_{\tau}(s_{\tau}, a_{\tau}) \;\middle|\; s_h = s}, \\
  Q_h^{\pi}(s, a) &:= \E[\pi,\P]{\textstyle\sum_{\tau=h}^{H} r_{\tau}(s_{\tau}, a_{\tau}) \;\middle|\; s_h = s,~ a_h = a}. \nonumber
\end{align}
For simplicity, we abuse the above notation a bit and write $V^{\pi}_1(\rho) := \E[s_1 \sim \rho]{V^{\pi}_1(s)}$. Further, we define the operator $[\P_h V](s,a) := \E[s' \sim \P_h(\cdot | s,a)]{V(s')}$. The (nominal) \textit{Bellman eqution} with respect to policy $\pi$ can be written as
\begin{equation*}
  Q^{\pi}_h(s,a) = \left[r_h + \P_{h+1} V^{\pi}_{h+1}\right](s,a)=: [\Bnom_h V^{\pi}_{h+1}](s,a)
\end{equation*}
where we define the Bellman update as an operator $\Bnom_h$.\footnote{Note that here the Bellman operator denotes the update from $V_{h+1}$ to $Q_h$, which is the same as in \cite{ma2023distributionally}, but may be different from some other literature.}

Given a policy $\pi$, the \textit{state occupancy measure} is defined by
\begin{equation}
  \rho^{\pi}_h(s) := \Prob[\pi, \P]{s_h = s \mid s_1 \sim \rho},
\end{equation}
and the \textit{state-action occupancy measure} is defined by
\begin{equation}
  d^{\pi}_h(s,a) := \Prob[\pi, \P]{s_h = s,~ a_h = a \mid s_1 \sim \rho}.
\end{equation}
Note that we always have $d^{\pi}_h(s,a) = \rho^{\pi}_h(s) \pi(a | s)$.

\paragraph{Low-rank MDPs.} An MDP $\M$ is said to have a low-rank representation, if there exists a \textit{feature map} $\bphi_h: \S \times \A \to \R^d$ and two \textit{factors} $\bmu_h: \S \to \R^d$, $\bnu_h \in \R^d$ for each step $h \in [H]$, such that for any $s, s' \in \S$ and $a \in \A$,
\begin{subequations}
\begin{align}
  \P_h(s' | s,a) &= \angl{\bphi_h(s,a), \bmu_h(s')},~ \forall h \in [H]\\
  r_h(s,a) &= \angl{\bphi_h(s,a), \bnu_h},~ \forall h \in [H].
\end{align}
\end{subequations}
We will denote the linear MDP by $\MDP(\bphi_h, \bmu_h, \bnu_h)$ when $H$, $\S$, $\A$ and $\rho$ are clear from context. Note that feature maps are allowed to be different across steps.

It is well-known that, in low-rank MDPs, $Q^{\pi}_h$ is also linearly-representable by the same feature map as
\begin{equation}
  Q^{\pi}_h(s,a)
  = [r_h(s,a) + \P_h V^{\pi}_{h+1}](s,a) \\
  = \angl[\Big]{\bphi_h(s,a), \underbrace{\textstyle \bnu_h + \sum_{s'} V^{\pi}_{h+1}(s') \bmu_h(s')}_{\bomega^{\pi}_h}},
\end{equation}
where $\bomega^{\pi}_h$ is the \textit{factor} for $Q$-function.

The following assumption is standard for low-rank MDPs in literature, which we will also follow by convention.
\begin{assumption}[bounded norms of features and factors]\label{assm:2-low-rank_MDP_bound}
  $\norm{\bphi_h(s,a)} \leq 1$, $\norm{\bnu_h} \leq \sqrt{d}$, $\norm{\sum_{s} V(s) \bmu_h(s)} \leq \sqrt{d}$ for any $s \in \S$, $a \in \A$ and $V: \S \to [0,H]$.
\end{assumption}



  \section{Robust Low-Rank MDP with Dual Perturbation and \titlemath{(\bxi, \bm{\eta})}-Rectangularity}

In this section, we will introduce a new robustness concept by exploiting the low-rank structure in MDPs, which avoids the $L_\infty$ perturbation, and thus enables the design of efficient algorithms for large state-action spaces.

\subsection{The Challenges of Robustness in Low-Rank MDPs}

As has been discussed above, the demand for robustness naturally arises when there is a sim-to-real gap. The traditional notion of robust MDPs imposes perturbations on $\P_h$ and $r_h$. For robust low-rank MDPs with \textit{dual perturbation}, it is reasonable to adapt the definition of uncertainty sets to be with respect to the feature maps and factors. Specifically, one may consider the family $\M$ consisting of low-rank MDPs $\MDP(\bphi_{1:H}, \bmu_{1:H}, \bnu_{1:H})$ centered around the nominal model $\MDP(\bphi_{1:H}^{\circ}, \bmu_{1:H}^{\circ}, \bnu_{1:H}^{\circ})$, such that for any $h \in [H]$,
\begin{subequations}\label{eq:3-standard_robust_MDP}
\begin{align}
  \norm{\bphi_h(s, a) - \bphi_h^{\circ}(s, a)} &\leq \r_{\phi,h},~ \forall (s,a) \in \S \times \A \\
  \norm{\bmu_h(s') - \bmu_h^{\circ}(s')} &\leq \r_{\mu,h},~ \forall s' \in \S \\
  \norm{\bnu_h - \bnu_h^{\circ}} &\leq \r_{\nu,h}, \\
  \bphi_h(s,a)^{\top} \bmu_h(\cdot) &\in \Delta(\A), \label{eq:3-standard_robust_MDP:4}
\end{align}
\end{subequations}
It is worth noting that \eqref{eq:3-standard_robust_MDP} requires the features and factors to form a valid linear MDP, so that \eqref{eq:3-standard_robust_MDP:4} is necessary since some perturbations in the ambiguity set will break the normalization condition, i.e., $\sum_{s'} \angl{\bphi_h(s, a), \bmu_h(s')} \neq 1$. 
We call $\M = \bigotimes_{h \in [H]} \M_h$ a \textit{$(\phi,\mu,\nu)$-rectangular ambiguity set}, where $\M_h$ is the stage ambiguity set at step $h$.

The objective of solving a robust low-rank MDP is to find its optimal robust policy $\pirob^*$, such that under all possible stage-wise perturbations in $\M$, the worst-case cumulative reward is maximized by $\pirob^*$. Formally, we define the (standard) robust value function of a given policy $\pi$ to be
\begin{equation}\label{eq:3-standard_robust_value}
  \Vrob^{\pi}_h(s) := \min\limits_{\begin{subarray}{c} (\bphi_{1:H}, \bmu_{1:H}, \bnu_{1:H}) \in \M,\\ \mathbb{P}_h = \angl{\bphi_h,\bmu_h}, r_h = \angl{\bphi_h,\bnu_h} \end{subarray}} \E[\pi, \mathbb{P}_h]{\sum_{\tau=h}^{H} r_{\tau}(s_{\tau},a_{\tau}) \;\middle|\; s_h = s},
\end{equation}
which represents the worst-case performance for a policy $\pi$ under all possible perturbations in the uncertainty set $\M$. The robust planning problem can then be formulated as
\begin{equation}
  \pirob^* := \arg\max_{\pi} \Vrob^{\pi}_1(\rho).
\end{equation}

To solve the robust low-rank MDP, a natural approach is to iteratively improve the policy with respect to its robust value, for which a robust policy evaluation scheme is needed. It is well-known that robust policy evaluation can be performed via robust dynamic programming \cite{iyengar2005robust}, which recursively updates the robust reward-to-go by Bellman equation. Formally, we define a robust Bellman update operator $\Brob_h$, such that given a robust $V$-function of policy $\pi$ at step $h+1$, i.e. $\Vrob^{\pi}_{h+1}$, the robust Bellman update is
\begin{align}\label{eq:3-robust_update}
  [\Brob_h \Vrob^{\pi}_{h+1}](s,a)
  &:= \min_{(\bphi_h,\bmu_h,\bnu_h) \in \M_h} \prn*{r_h(s,a)+ \sum_{s'} \P_h(s'|s,a) \Vrob^{\pi}_{h+1}(s')} \nonumber\\
  &=  \min_{(\bphi_h,\bmu_h,\bnu_h) \in \M_h} \angl*{\bphi_h(s,a), \bnu_h + \sum_{s'} \Vrob^{\pi}_{h+1}(s') \bmu_h(s')},
\end{align}
For boundary conditions, we always regard $\Vrob^{\pi}_{H+1}(\cdot) \equiv 0$.

However, the algorithm is computationally undesirable. 
On the one hand, for each time step $h$, we need to solve an independent optimization problem for each $(s,a)$ pair, which each involves $\Theta(dS)$ optimization variables, breaking the benefits of low-rank structure and thus making the algorithm prohibitively slow for large state-action spaces. On the other hand, the constraint of $(s,a)$-rectangular ambiguity set makes optimization even harder due to its incompatibility with low-rank structure --- to guarantee that $\angl{\bphi_h(s,a), \bmu_h(\cdot)}$ lies in the probability simplex for each $(s,a)$ pair, we need to introduce $\Theta(SA)$ additional constraints that make the feasible region highly nonconvex. As a result, the naive approach to directly adapt the standard ambiguity set definition and perform robust policy evaluation fails to exploit the advantage of introducing low-rank representations, but rather, only adds to the complexity of optimization and magnifies the drawbacks of representation.

As discussed above, in some recent papers like \citet{goyal2023robust, ma2023distributionally}, a new robustness concept called $d$-rectangularity is proposed specifically for low-rank MDPs, along with new algorithms for finding the optimal policies. However, their ambiguity set definition is restrictive, in that it requires a specific soft state-aggregate structure for the low-rank representation, assumes the ambiguity set to be convex or KL-divergence regularized, and cannot handle perturbations of feature maps. We would like to emphasize that no existing methods are able to handle the dual perturbations on both the feature maps and the factors in generic low-rank MDPs.

The challenges mentioned above motivate us to find an alternative formulation of robustness that works better with the low-rank representation structure and is more computationally-friendly.

\subsection{The Proposed Robust Low-Rank MDP}\label{sec:3-2-robust_definition}

In essence, in order to improve robustness, we only need to perturb the system dynamics around the nominal model in any procedure that ``solves'' the nominal problem. For this purpose, note that in the nominal MDP, the ultimate objective $V^{\pi}_1(\rho)$ can be expanded in terms of $Q^{\pi}_h(\cdot,\cdot)$ as
\begin{equation}
  V^{\pi}_1(\rho) = \sum_{\tau < h} \E[(s_{\tau},a_{\tau}) \sim d^{\pi}_{\tau}]{r_{\tau}(s_{\tau}, a_{\tau})} + \E[(s_h,a_h) \sim d^{\pi}_h]{Q^{\pi}_h(s_h, a_h)}.
\end{equation}
Therefore, the way $Q^{\pi}_h$ appears in $V^{\pi}_1$ is only through the expectation $\E[(s_h,a_h) \sim d^{\pi}_h]{Q^{\pi}_h(s_h, a_h)}$. For low-rank MDPs, we shall further expand that to
\begin{equation}\label{eq:3-nominal_update}
  \angl[\big]{\E[(s_h,a_h) \sim d^{\pi}_h]{\bphi_h(s_h,a_h)}, \bomega_h},
\end{equation}
where $\bomega_h := \bnu_h + \sum_{s'} \Vnom^{\pi}_{h+1}(s') \bmu_h(s')$ is the parameter for the nominal $Q$-function, and the expectation is over $(s_h,a_h) \sim d^{\pi}_h$, the state-action occupancy measure under policy $\pi$ at step $h$. Here we utilize the linearity of expectations, as the factor $\bomega_h$ is independent from $(s_h,a_h)$. 

Suppose the perturbations on $\bphi_h(s,a)$, $\bmu_h(s')$ and $\bnu_h$ are denoted by $\bdelta_{\phi,h}(s, a)$, $\bdelta_{\mu,h}(s')$ and $\bdelta_{\nu,h}$, respectively. Plug them into \eqref{eq:3-nominal_update}, and we have 
\begin{equation}\label{eq:3-perturbed_nominal_update}
  \angl[\big]{\E[(s_h,a_h) \sim d^{\pi}_h]{\bphi_h^{\circ}(s_h,a_h) + \bdelta_{\phi,h}(s,a)}, \bomega_h + \bxi_h}
  = \angl[\big]{\E[(s_h,a_h) \sim d^{\pi}_h]{\bphi_h^{\circ}(s_h,a_h) } + \bm{\eta}_h, \bomega_h + \bxi_h} 
\end{equation}
where $\bomega_h := \bnu^{\circ}_h + \sum_{s'} \Vnom^{\pi}_{h+1}(s') \bmu^{\circ}_h(s')$, $\bxi_h := \bdelta_{\nu,h} + \sum_{s'} \Vnom^{\pi}_{h+1}(s') \bdelta_{\mu,h}(s')$, and $\bm{\eta}_h := \E[(s_h,a_h) \sim d^{\pi}_h]{\bdelta_{\phi,h}(s, a)}$. Details of this change-of-variable can be found in \Cref{sec:appx-A-1-transformation}.
The obtained equation~\eqref{eq:3-perturbed_nominal_update} reveals the essential effect the dual perturbation over $\bnu$, $\bmu$ and $\bphi$ has on the value of the policy, which inspires us to consider the ``effective'' perturbation over $\bxi_h$ and $\bm{\eta}_h$ for computational efficiency.

Now we are ready to formally define our novel robust policy evaluation scheme. Given policy $\pi$, define the following recursively: for the terminal value, set $\Vrlx^{\pi}_{H+1}(\cdot) \equiv 0$; for the recursive update at time step $h \in [H]$, given $\Vrlx^\pi_{h+1}(s')$ for step $h+1$, we first solve an optimization problem
\begin{equation}\label{eq:3-effective_robust_update}
  \min_{(\bxi_h, \bm{\eta}_h) \in \widehat{\M}_h} \angl[\big]{\E[(s_h,a_h) \sim d^{\pi}_h]{\bphi_h^{\circ}(s_h,a_h)} + \bm{\eta}_h, \bomega_h + \bxi_h},
\end{equation}
which can be viewed as a perturbation of \eqref{eq:3-nominal_update} around the nominal dynamics. The recursion happens within the calculation of $\bomega_h := \bnu^{\circ}_h + \sum_{s'} \Vrlx^{\pi}_{h+1}(s') \bmu^{\circ}_h(s')$, where $\Vrlx^{\pi}_{h+1}(\cdot)$ from last iteration is used. Note that here $d^{\pi}_h$ still refers to the state-action occupancy measure in the \textit{nominal} model. We say $\widehat{\M}_h$ is a \textit{$(\xi,\eta)$-rectangular ambiguity set}, if it is rectangular in terms of $(\bxi_h, \bm{\eta}_h)$ as
\begin{equation}\label{eq:3-effective_ambiguity_set}
  \widehat{\M}_h := \set[\big]{(\bxi_h, \bm{\eta}_h)}[\norm{\bxi_h} \leq \r_{\xi,h}, \norm{\bm{\eta}_h} \leq \r_{\eta,h}],
\end{equation}
where $(\r_{\xi,h}, \r_{\eta,h})$ are called the radii of perturbation. Further, we define $\widehat{\M} := \bigotimes_{h \in [H]} \widehat{M}_h$.

With the solution $(\bxi^*_h, \bm{\eta}^*_h)$ in hand, we proceed to calculate the robust $Q$-functions under the new robustness concept. Although the optimization problem \eqref{eq:3-effective_robust_update} does not produce individual perturbed features, as \eqref{eq:3-perturbed_nominal_update} suggests, we may simply perturb each feature $\bphi^{\circ}_h(\cdot,\cdot)$ by the same amount $\bm{\eta}^*_h$. Therefore, the new Bellman update can be written as
\begin{equation}\label{eq:3-effective_robust_operator}
  [\Brlx^{\pi}_h \Vrlx^{\pi}_{h+1}](s,a) := \angl*{\bphi_h^{\circ}(s,a) + \bm{\eta}^*_h, \bomega_h + \bxi^*_h}.
\end{equation}
Note that here we explicitly mark the policy behind an Bellman update operator in its superscript, since the policy is implicitly involved when we solve \eqref{eq:3-effective_robust_update}. Then the low-rank robust $V$-function can be recovered by
\begin{equation}\label{eq:3-effective_robust_value}
  \Vrlx^{\pi}_h(s) = \angl[\big]{\pi_h(\cdot | s), [\Brlx^{\pi}_h \Vrlx^{\pi}_{h+1}](s, \cdot)}.
\end{equation}
For the sake of convenience we also define the low-rank robust $Q$-function as $\Qrlx^{\pi}_h := [\Brlx_h \Vrlx^{\pi}_{h+1}]$. The objective of the robust planning problem is to find the optimal policy that maximizes the robust value at the initial step, namely
\begin{equation}\label{eq:3-effective_robust_policy}
  \pirlx^* := \arg\max_{\pi} \min_{\MDP \in \widehat{\M}} \Vrlx^{\pi}_1(\rho).
\end{equation}
The new robustness concept can be interpreted as an \textit{implicit} step-wise independent pseudo-MDP\footnote{Pseudo-MDPs are MDP-like processes that allow transition probabilities to lie out of the probability simplex. See \citet{yao2014pseudo} for detailed definitions and properties.} perturbation around the nominal MDP, which is done through $2H$ \textit{effectively equivalent} perturbation vectors $\bxi_{1:H}$ and $\bm{\eta}_{1:H}$.

\begin{remark}
  Note that, in theory, we may write out a set $\widehat{\M}$ that contains exactly the $(\xi,\eta)$ pairs corresponding to some valid MDP perturbation around the nominal model. However, for the sake of computational simplicity, we choose to relax the ambiguity set to include pseudo-MDPs.
\end{remark}

\subsection{Rationale of the Proposed Low-rank Robustness with \titlemath{(\bxi, \bm{\eta})}-Rectangularity}

We proceed to present properties and examples to justify and promote understanding of the proposed robustness concept.

\paragraph{Relationship with Nominal and Standard Robust Updates.} One may wonder how the new policy evaluation scheme is connected with the standard robustness through representations in~\eqref{eq:3-standard_robust_MDP}, and further, how the new robust values relate to the standard ones. For this purpose, the readers should be reminded of the standard robust Bellman update operator $\Brob_h$ defined in \eqref{eq:3-robust_update}, and our new robust Bellman update operator $\Brlx_h$ defined in \eqref{eq:3-effective_robust_update}. In addition, we point out that we may always select $\r_{\xi,h}$ and $\r_{\eta,h}$ so that the $(\xi, \eta)$-ambiguity set induced by $\M$ is a subset of $\widehat{\M}$. \Cref{sec:appx-A-1-transformation} contains more information about this transformation, and we will simply assume this by default.

It turns out that we can show the following property.
\begin{theorem}\label{thm:3-value_relation}
  Suppose the $(\xi, \eta)$-ambiguity set induced by $\M$ is a subset of $\widehat{\M}$. Then for any step $h \in [H]$ we have:
  \begin{enumerate}
    \item \textit{Ordinal relation}: $\Qrlx^{\pi}_h(s,a) \leq \Qrob^{\pi}_h(s,a) \leq \Qnom^{\pi}_h(s,a)$, and $\Vrlx^{\pi}_h(s) \leq \Vrob^{\pi}_h(s) \leq \Vnom^{\pi}_h(s)$;

    \item \textit{Bounded gap}: $\Vrob^{\pi}_h(s) - \Vrlx^{\pi}_h(s) \leq \Vnom^{\pi}_h(s) - \Vrlx^{\pi}_h(s) \leq \sum_{\tau \geq h} \prn[\big]{2 \r_{\eta,\tau} \sqrt{d} + (1+\r_{\eta,\tau}) \r_{\xi,\tau}}$.
  \end{enumerate}
\end{theorem}
The above theorem clearly states that the proposed robust policy evaluation scheme can be viewed as a relaxation of the standard robust evaluation scheme (given appropriate radii). Indeed, our scheme leads to a pessimistic evaluation of the policy, and thus yields more conservative policies. 

It seems that the above bound is loose in that the gap between $\Vnom^{\pi}_1$, $\Vrob^{\pi}_1$ and $\Vrlx^{\pi}_1$ is in the order of $\Theta(H)$. However, the following example shows that this is actually the best bound we can expect.

\begin{example}[string guessing]
  Consider a string guessing game with the answer set to be an $m$-bit binary string. Without loss of generality, let  $11 \cdots 1$ be the answer. There are two actions, i.e. $\A = \set{a_0,a_1}$. The game proceeds in a bit-wise manner, and the transient states $s_{1:m}$ are used to record the progress. There are two absorptive states: $s_-$ for an error on any bit, which yields a reward of 0 for each of the remaining steps; $s_+$ for success on all bits, which yields a reward of 1 for each of the remaining steps. The MDP is illustrated in \Cref{fig:6-1-string_guessing} below, where transitions are deterministic (as indicated by the arrows), and all rewards are 0 except for the self-loop at $s_+$.

  \begin{figure}[ht]
    \centering
      \tikzset{snode/.style = {draw=black, shape=circle, line width=1.0pt, inner sep=0pt, minimum width=16pt}}
  \begin{tikzpicture}[%
    every path/.style = {draw=black, line width=0.5pt, ->, >={Stealth}}%
  ]
    \node[snode] (s0) at (80pt, 50pt) {$s_-$};
    \node[snode] (s1) at (0pt, 0pt) {$s_1$};
    \node[snode] (s2) at (40pt, 0pt) {$s_2$};
    \node[snode] (s3) at (80pt, 0pt) {$s_3$};
    \node (s4) at (120pt, 0pt) {$\cdots$};
    \node[snode] (sm) at (160pt, 0pt) {$s_m$};
    \node[snode] (sp) at (200pt, 0pt) {$s_+$};

    \draw (s1) edge node[above]{$a_1$} (s2);
    \draw (s2) edge node[above]{$a_1$} (s3);
    \draw (s3) edge node[above]{$a_1$} (s4);
    \draw (s4) edge node[above]{$a_1$} (sm);
    \draw (sm) edge node[above]{$a_1$} (sp);
    \draw (sp) edge[loop above] node[above]{\scriptsize\color{blue} $r=1$} (sp);

    \draw (s1) edge node[left]{$a_0$} (s0);
    \draw (s2) edge node[right]{$a_0$} (s0);
    \draw (s3) edge node[right]{$a_0$} (s0);
    \draw (sm) edge node[right]{$a_0$} (s0);
    \draw (s0) edge[loop above] node[left]{} (s0);
  \end{tikzpicture}\vspace{-8pt}
    \caption{MDP diagram for the string guessing game.}\label{fig:6-1-string_guessing}
  \end{figure}
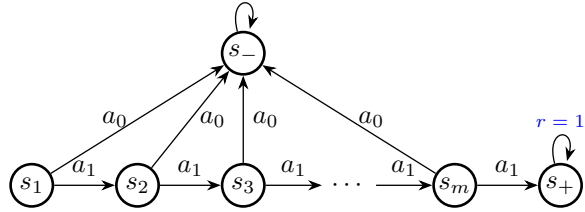

  To formulate robust MDPs around the nominal model, suppose that the transition probabilities are subject to uncertainties of at most $\delta$. Then the standard $(\phi,\mu,\nu)$-ambiguity set $\M$ shall be specified by $\r_{\phi,h} = \r_{\nu,h} = 0$ and $\r_{\mu,h} = \delta$, while the $(\xi,\eta)$-ambiguity set $\widehat{\M}$ shall be specified by $\r_{\xi,h} = (H-h)\delta$ and $\r_{\eta,h} = 0$. It can be verifed that $\widehat{\M}$ is a relaxation of $\M$ (see \Cref{sec:appx-A-1-transformation}).

  Consider a policy that always takes action $a_1$. The nominal, standard robust and the new robust values shall be calculated as $\Vnom^{\pi}_1(s_1) = H-m$, $\Vrob^{\pi}_1(s_1) = (1-\delta)^m (H-m)$, and $\Vrlx^{\pi}_1(s_1) = (H - m) - \sum_{h=1}^{m} (H-h)\delta$, respectively.
  
  With these values in hand, we first verify their ordinal relations, which follows from Bernoulli's inequality as
  \begin{equation}
    \Vnom^{\pi}_1(s_1)
    > \Vrob^{\pi}_1(s_1)
    > \Vrlx^{\pi}_1(s_1).
  \end{equation}
  Meanwhile, it can be shown by Taylor expansion that
  \begin{align}
    \Vrob^{\pi}_1(s_1) - \Vrlx^{\pi}_1(s_1)
    &= \frac{m(m+1)}{2} \delta + \o(\delta), \\
    \Vnom^{\pi}_1(s_1) - \Vrob^{\pi}_1(s_1)
    &= (H-m)m \delta + \o(\delta).
  \end{align}
  Therefore, with sufficiently small $\delta$, $\Vrob^{\pi}_1(s_1)$ would be much closer to $\Vrlx^{\pi}_1(s_1)$ when $H \gg m$, but become much closer to $\Vnom^{\pi}_1(s_1)$ when $H = \Theta(m)$, which implies that in general we cannot expect anything better than \Cref{thm:3-value_relation}.

  Details of this example can be found in \Cref{sec:appx:A-2-value_relation}.
\end{example}

\paragraph{Robustness Induced by Low-rank Robust MDPs.} We show by another simple example that the optimal robust policy for the proposed low-rank robustness concept indeed displays robust behavior to a certain extent.
\begin{example}[gamble-or-guarantee]
  Consider the following ``gamble-or-guarantee'' game, where the agent is required to enter one of the two branches: a no-risk ``guarantee'' branch (taking action $a_0$) including a single absorptive state $s_{\alpha}$ to receive a constant reward $\alpha$ from then on, and a risky ``gamble'' branch (taking action $a_1$) that includes a potentially transient state $s_1$ to receive rewards of $1$, at the risk of permanently falling into the 0-reward absorption state $s_0$. The MDP is illustrated in \Cref{fig:4-2-risk_seeking} below.

  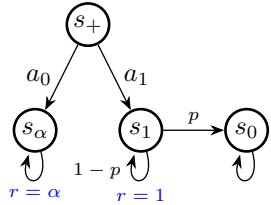
\begin{figure}[ht]
    \centering
      \tikzset{snode/.style = {draw=black, shape=circle, line width=1.0pt, inner sep=0pt, minimum width=16pt}}
  \tikzset{apath/.style = {draw=black, line width=0.5pt, ->, >={Stealth}}}
  \begin{tikzpicture}
    \node[snode] (si) at (0pt, 40pt) {$s_+$};
    \node[snode] (sa) at (-20pt, 0pt) {$s_{\alpha}$};
    \node[snode] (s1) at (20pt, 0pt) {$s_1$};
    \node[snode] (s0) at (60pt, 0pt) {$s_0$};

    \draw[apath] (si) edge node[left]{$a_0$} (sa);
    \draw[apath] (si) edge node[right]{$a_1$} (s1);
    \draw[apath] (s1) edge node[above=-1pt]{\scriptsize $p$} (s0);
    
    \draw[apath] (sa) edge[loop below] node[below]{\scriptsize\color{blue} $r=\alpha$} (sa);
    \draw[apath] (s1) edge[loop below] node[left=4pt,yshift=3pt]{\scriptsize $1-p$} node[below]{\scriptsize\color{blue} $r=1$} (s1);
    \draw[apath] (s0) edge[loop below] node[below]{} (s0);
  \end{tikzpicture}\vspace{-8pt}
    \caption{MDP diagram for the gamble-or-guarantee game.}\label{fig:4-2-risk_seeking}
  \end{figure}

  Direct computation shows that, when $H$ is sufficiently large, the optimal nominal policy is to take action $a_1$ at the initial state $s_+$. Meanwhile, we also have $\Vrlx^*_h(s_{\alpha}) = \frac{(1-\delta)\alpha}{\delta} \prn[\big]{1 - (1-\delta)^{H-h+1}}$ and $\Vrlx^*_h(s_1) \leq \frac{1-p-\delta}{p+\delta}$, so that $\Vrlx^*_2(s_1) < \Vrlx^*_2(s_{\alpha})$ when $\delta < \frac{p\alpha}{2(1-p)}$, in which case the optimal robust policy is to take action $a_0$ at intial state $s_+$. 
  
  The above example indicates that, under appropriate perturbation radii, the policy obtained by solving \eqref{eq:3-effective_robust_policy} indeed displays certain level of robustness in behavior.
  
  Details of this example can be found in \Cref{sec:appx-A-3-robustness_example}.
\end{example}
  \section{R\tsup{2}PG: Representation Robust Policy Gradient}

With the new robustness concept in hand, now we shall present our algorithm that iteratively solves for the optimal robust policy $\pirlx^*$ as defined in \eqref{eq:3-effective_robust_policy}.

\begin{algorithm*}[t]
  \caption{R\tsup{2}PG: \textbf{R}epresentation \textbf{R}obust \textbf{P}olicy \textbf{G}radient}\label{alg:4-main}
  \begin{algorithmic}[1]
    \State Initialize $\pi^1_h(\cdot | s) \gets \mathsf{Unif}(\A)$.
    \For{$k = 1,2,\cdots,K$}
      \State Initialize $\Vrlx^k_{H+1}(s) \gets 0$.
      \State Compute $d^{\pi^k}_h$ via recursion: $d^{\pi^k}_1(s,a) \gets \rho(s) \pi(a | s)$, $d^{\pi^k}_{h+1}(s',a') \gets \sum_{s,a} d^{\pi^k}_h(s,a) \P^{\circ}_h(s' | s,a) \pi(a' | s')$.
      \For{$h = H,H-1,\cdots,1$}
        \State Compute $\bomega^{\circ,k}_h \gets \bnu^{\circ}_h + \sum_{s'} \Vrlx^k_{h+1}(s') \bmu^{\circ}_h(s')$, and solve the following program for $\bxi^k_h$, $\bm{\eta}^k_h$:
        \begin{equation}\label{eq:4-policy_eval_optim}
          \min_{(\bxi^k_h, \bm{\eta}^k_h) \in \widehat{\M}} \angl*{\E[(s,a) \sim d^{\pi^k}_h]{\bphi^{\circ}_h(s,a)} + \bm{\eta}^k_h, \bomega^{\circ,k}_h + \bxi^k_h}.
        \end{equation}
        \State Perform feature update: $\bphi^k_h(s,a) \gets \bphi^{\circ}_h(s,a) + \bm{\eta}^k_h$, $\bomega^k_h \gets \bomega^{\circ,k}_h + \bxi^k_h$.
        \State Update value functions: $\Qrlx^k_h(s,a) \gets \angl{\bphi^k_h(s,a), \bomega^k_h}$, $\Vrlx^k_h(s) \gets \sum_{a} \pi^k_h(a | s) \Qrlx^k_h(s,a)$.
        \State Use Natural Policy Gradient to update the policy: $\pi^{k+1}_h(a | s) \propto \pi^k_h(a | s) \cdot \exp\prn[\big]{\alpha \Qrlx^{\pi}_h(s,a)}$.
      \EndFor
    \EndFor
    \State \Return $\pi^{\mathrm{out}} \sim \mathsf{Unif}(\pi^{1:K})$
  \end{algorithmic}
\end{algorithm*}

\subsection{Algorithm Design}

The proposed algorithm, \textbf{R}epresentation \textbf{R}obust \textbf{P}olicy \textbf{G}radient (R\tsup{2}PG), can be found in \Cref{alg:4-main}. Overall, it follows the standard evaluation-improvement protocol, except that it uses the new robust policy evaluation scheme proposed in \Cref{sec:3-2-robust_definition}. The R\tsup{2}PG algorithm consists of two main components for each step in each iteration:
\begin{enumerate}
  \item \textit{Policy evaluation (line 6-8)}. The algorithm follows the effective robust policy evaluation scheme formulated in \Cref{sec:3-2-robust_definition}. At time step $h$, it first solves \eqref{eq:4-policy_eval_optim} (or equivalently \eqref{eq:3-effective_robust_update}), to get the optimal $(\bxi^*_h, \bm{\eta}^*_h)$ \textit{(line 6)}, and then updates the robust $Q$- and $V$-functions according to \eqref{eq:3-effective_robust_operator} and \eqref{eq:3-effective_robust_value} with identical perturbations $\bm{\eta}^*_h$ around each nominal feature $\bphi^{\circ}_h(\cdot,\cdot)$ \textit{(line 7-8)}. Note that \eqref{eq:3-effective_robust_update} is a non-convex optimization problem, for which the global optimum is hard to find using a general-purpose optimizer. We will discuss how to efficiently solve this problem in the next section by reducing it to an SDP.

  \item \textit{Policy improvement (line 9)}. To update the policy, we use the Natural Policy Gradient (NPG) algorithm that is widely-used in literature \citep{agarwal2021theory, cen2022fast,mei2021leveraging}. Given the robust $Q$-function $\Qrlx^{\pi^k}_h$ of policy $\pi^k$, the update rule is given by
  \begin{equation}
    \pi^{k+1}_h(a | s) \propto \pi^k_h(a | s) \cdot \exp\prn[\big]{\alpha \Qrlx^{\pi^k}_h(s,a)}
  \end{equation}
  for some step size $\alpha > 0$. We would like to point out that NPG in the episodic setting can be interpreted as maintaining an Online Mirror Descent instance at each $(h,s) \in [H] \times \S$ to solve an \textit{expert-advice} problem, which is a well-studied area in online learning literature \cite{orabona2019modern}. Details regarding the interpretation and its learning regret can be found in \Cref{sec:appx-NPG}.
\end{enumerate}

For explanatory purposes, we temporarily ignore the computational difficulties of finding the global minimum of a non-convex program \eqref{eq:4-policy_eval_optim}, as well as summing over the whole state-action space to exactly calculate $d^{\pi^k}_h$ and $\bomega^{\circ,k}_h$. The issue of generalizing the algorithm to MDPs with large state-action spaces will be discussed in later sections.

\subsection{Computational Considerations}

A crucial computational bottleneck in \Cref{alg:4-main} is to solve the optimization problem \eqref{eq:4-policy_eval_optim} for the proposed robust Bellman update. In fact, as has been indicated above, \eqref{eq:4-policy_eval_optim} is a non-convex program, for which the global optimum may be hard to find by a general-purpose gradient-based optimizer. Fortunately, we can reduce it to a constrained \textit{Semi-Definite Programming (SDP)} problem that is computationally more approachable.

For the ease of exposition, consider the general optimization problem with the same structure as \eqref{eq:4-policy_eval_optim}, i.e.
\begin{subequations}\label{eq:4-2-optim_inner_product}
\begin{align}
  \min_{\bm{x}, \bm{y}}\quad& \angl*{\bm{a}+\bm{x}, \bm{b}+\bm{y}},\qquad \\
          \textrm{s.t.}\quad& \norm{\bm{x}} \leq R_x,~ \norm{\bm{y}} \leq R_y.
\end{align}
\end{subequations}
Here $\bm{x}$ and $\bm{y}$ correspond to $\bxi_h$ and $\bm{\eta}_h$ in \eqref{eq:4-policy_eval_optim}, respectively. To describe how to reduce \eqref{eq:4-2-optim_inner_product} to a constrained SDP, we first rewrite it as a quadratic program. Let $\bm{z} := [\bm{x}^{\top}, \bm{y}^{\top}]^{\top}$, so that the objective function and constraints can all be rewritten as a quadratic function in $\bm{z}$, i.e.
\begin{subequations}\label{eq:4-2-optim_QC2QP}
\begin{align}
  \min_{\bm{z}}\quad& \bm{z}^{\top} \underbrace{\begin{bmatrix}
    0 & \frac{1}{2} \bm{I} \\
    \frac{1}{2}\bm{I} & 0
  \end{bmatrix}}_{\bm{\bm{A}}} \bm{z} + 2 {\underbrace{\begin{bmatrix}
    \tfrac{1}{2}\bm{b} \\ \tfrac{1}{2}\bm{a}
  \end{bmatrix}}_{\bm{\beta}}}^{\top} \bm{z} + \underbrace{\angl{\bm{a}, \bm{b}}}_{c}, \\
  \textrm{s.t.}\quad& \bm{z}^{\top} \underbrace{\begin{bmatrix}
    \bm{I} & \\
    & 0
  \end{bmatrix}}_{\bm{\bm{A}}_x} \bm{z} \leq R_x^2,\quad
  \bm{z}^{\top} \underbrace{\begin{bmatrix}
    0 & \\
    & \bm{I}
  \end{bmatrix}}_{\bm{\bm{A}}_y} \bm{z} \leq R_y^2.
\end{align}
\end{subequations}
In this way, \eqref{eq:4-2-optim_QC2QP} becomes a \textit{Quadratic Program with \textit{Two} Quadratic Constraints (QC2QP)} by nature, which has been widely studied in literature \citep{ai2009strong, chen2021optimality}. In fact, we have the following equivalence.

\begin{theorem}[Reduction]\label{thm:4-2-reduction}
  Consider the following SDP
  \begin{subequations}\label{eq:4-2-QC2QP_dual-SDP}
  \begin{align}
    \min_{\bm{X} \in S^{2d+1}}&\quad \tr(\bm{C} \bm{X}) \label{eq:4-2-QC2QP_dual-SDP:1}\\
    \mathrm{s.t.}&\quad \tr(\bm{C}_x \bm{X}) \leq 0,\quad
                        \tr(\bm{C}_y \bm{X}) \leq 0, \label{eq:4-2-QC2QP_dual-SDP:2}\\
                 &\quad \tr(\bm{C}_0 \bm{X}) = 1,\quad
                        \bm{X} \succeq 0, \label{eq:4-2-QC2QP_dual-SDP:3}
  \end{align}
  \end{subequations}
  where $S^n$ denotes the set of $n$-by-$n$ real-symmetric matrices, and the constant matrices are defined as
  \begin{equation}
    \bm{C} := \begin{bmatrix}
      \bm{A} & \bm{\beta} \\
      \bm{\beta}^{\top} & c
    \end{bmatrix},\quad
    \bm{C}_0 := \begin{bmatrix}
       \bm{0} & \\
       & 1
    \end{bmatrix},\quad
    \bm{C}_x := \begin{bmatrix}
       \bm{A}_x & \\
       & -R_x^2
    \end{bmatrix},\quad
    \bm{C}_y := \begin{bmatrix}
       \bm{A}_y & \\
       & -R_y^2
    \end{bmatrix}.
  \end{equation}
  Let $\bm{X}^*$ be the optimal solution of \eqref{eq:4-2-QC2QP_dual-SDP}. Then we have $\bm{z}^* := \bm{X}^*_{1:2d,2d+1}$, i.e. the first $2d$ entries of the last column in $\bm{X}^*$, is the optimal solution of \eqref{eq:4-2-optim_QC2QP}.
\end{theorem}

In this way, we reduce the non-convex problem to a constrained SDP, for which efficient solvers are known to exist \citep{hazan2008sparse, gartner2012approximation}.

\subsection{Generalize to Large State-Action Spaces}\label{sec:large_state_space}

We would also like to discuss how our R\tsup{2}PG algorithm can be generalized to handle large state-action spaces. Note that the challenges lie in the computation of $d^{\pi^k}_h$ and $\bomega^{\circ,k}_h$ that involves summing over the state-action spaces, for which we shall propose a few alternative methods.

Throughout this section, we assume that we have access to a sampling oracle that collects $N$ trajectories $\set{(s^i_h, a^i_h, r^i_h)_{h \in [H]}}[i \in [N]]$ from the the \textit{nominal} MDP.

\paragraph{Estimating $\bm{\omega^{\circ,k}_h}$.} Similar to existing algorithms for linear MDPs \cite{bradtke1996linear,jin2019LSVI-UCB}, we estimate $\bomega^{\circ,k}_h$ by least-squares regression (with adjustable $\lambda > 0$):
\begin{equation}
  \bomega^{\circ,k}_h \gets \min_{\bomega} \textstyle\sum_{i=1}^{N} \E[\pi]{\norm[\big]{r^i_h + \Qrlx^k_{h+1}(s^i_{h+1}, \cdot) - \angl{\bphi^{\circ}_h(s^i_h, a^i_h), \bomega}}^2 + \lambda \norm{\bomega}^2}.
\end{equation}
In this way we have a sample-based estimator for $\bomega^{\circ,k}_h$ that can be computed by Stochastic Gradient Descent (SGD).

\paragraph{Approximating the Solution of (\ref{eq:4-policy_eval_optim}).} Note that we do not necessarily need to compute $d^{\pi^k}_h$ --- we can use any method that approximates the solution of \eqref{eq:4-policy_eval_optim}. Here we propose two different approaches to do this.
\begin{itemize}
  \item \textit{Monte-Carlo estimation of the averaged feature.} Assuming the same sampling oracle as above, we shall simply replace $\E[(s,a) \sim d^{\pi^k}_h]{\bphi^{\circ}_h(s,a)}$ with its Monte-Carlo estimation $\frac{1}{N} \sum_{i=1}^{N} \bphi^{\circ}_h(s^i_h, a^i_h)$, and solve the resulting optimization problem using the SDP reduction above.

  \item \textit{Perform SGD for regularized objective.} Alternatively, we may convert the constrained program to an unconstrained program with constraint-induced regularizers, i.e. 
  \begin{equation}
    \min_{\bxi^k_h, \bm{\eta}^k_h} \E[s, a]{\angl[\big]{\bphi^{\circ}_h(s,a) + \bm{\eta}^k_h, \bomega^{\circ,k}_h + \bxi^k_h}}    + \lambda_{\xi} \prn[\big]{\norm{\bxi^k_h}^2 - R_{\xi,h}^2} + \lambda_{\eta} \prn[\big]{\norm{\bm{\eta}^k_h}^2 - R_{\eta,h}^2},  
  \end{equation}
  and use SGD with samples $(s,a)\sim \set{(s^i_h, a^i_h) \mid i \in [N]}$ to approximate the solution of the regularized program. 
\end{itemize}
  \section{Theoretical Analysis}\label{sec:5-analysis}

In this section, we present the convergence guarantee for our R\tsup{2}PG algorithm. Specifically, we would like to show that the robust value of the output policy $\pi^{\mathrm{out}}$ is close to that of the optimal robust policy $\pirlx^*$ defined in ~\eqref{eq:3-effective_robust_policy}. To this end, we have the following bound.

\begin{theorem}[Convergence]\label{thm:5-convergence}
  Under \Cref{assm:2-low-rank_MDP_bound}, by running \Cref{alg:4-main} with $\alpha = \sqrt{2\log A / (KH^2)}$, the robust $V$-function of  $\pi^{\mathrm{out}}$ satisfies
  \begin{equation}
    \E[\pi^{\mathrm{out}}]{\Vrlx_1^*(\rho) - \Vrlx_1^{\pi^{\mathrm{out}}}(\rho)}
    \leq \sqrt{\frac{2H^4 \log A}{K}} + \sum_{h=1}^{H} \prn[\big]{2\r_{\xi,h} (1+\r_{\eta,h}) + 6\r_{\eta,h} \sqrt{d}}. 
  \end{equation}
\end{theorem}

\paragraph{Proof sketch.} We use a similar proof structure as in \citet{liu2023optimistic}, which consists of three main steps.

\textit{Step 1: Quasi-contraction property.} The first key observation is that the new robust Bellman operator satisfies a \textit{quasi}-contraction property, which is similar to the contraction of nominal and standard robust Bellman operators \citep{iyengar2005robust}, but is also different in a way that it only holds in expectation over $d^{\pi}_h$.

\begin{lemma}\label{thm:5-contraction}
  For any $V$-functions $V, V': \S \to \R$ and any policy $\pi$, we have
  \begin{equation}
    \E[(s,a) \sim d^{\pi}_h]{[\Brlx^{\pi}_h V](s,a) - [\Brlx^{\pi}_h V'](s,a)}
    \leq \E[s' \sim \rho^{\pi}_{h+1}]{V(s') - V'(s')} + 2 \r_{\eta,h} \sqrt{d}.
  \end{equation}
\end{lemma}

\textit{Step 2: Extended performance difference lemma.} We proceed to prove an extended Performance Difference Lemma that is similar to its counterpart in \citet{efroni2020optimistic}, which builds upon the quasi-contraction property.

\begin{lemma}\label{thm:5-performance_difference}
  For any policies $\pi$ and $\pi'$, we have
  \begin{align}\label{eq:5-performance_difference:e1}
    &\Vrlx_1^{\pi}(\rho) - \Vrlx_1^{\pi'}(\rho)
    \leq \sum_{h=1}^{H} \mathbb{E}_{\pi, \P^{\circ}} \Bigl[ \angl*{\Qrlx_h^{\pi'}(s_h,\cdot), \pi_h(\cdot | s_h) - \pi'_h(\cdot | s_h)} \nonumber\\
    &\hspace{9em} {}+ \angl*{[\Brlx^{\pi}_h \Vrlx_{h+1}^{\pi'}](s_h,\cdot) - \Qrlx_h^{\pi'}(s_h,\cdot), \pi_h(\cdot | s_h)} \Bigr] + 2 \sum_{h=1}^{H} \r_{\eta,h} \sqrt{d}.
  \end{align}
\end{lemma}

\textit{Step 3: Plugging in the upper bounds.} To complete the proof, we average \eqref{eq:5-performance_difference:e1} over the $K$ policies obtained in $K$ episodes. To bound the first term in the expectation in \eqref{eq:5-performance_difference:e1}, we plug in the regret bound for NPG (\Cref{thm:C-2-expert_regret_RL}); the second term can be bounded by a technical lemma (\Cref{thm:C-3-operator_difference}) that characterizes the difference between robust Bellman updates with respect to different policies.

Details of the proof are deferred to \Cref{sec:appx-C-proof}.
  \section{Numerical Simulation}

We study the numerical performance of our R\tsup{2}PG algorithm via simulations on a toy model. The setup is illustrated in \Cref{fig:6-simulation_MDP} (details deferred to \Cref{sec:appx-D-simulation}), where at each step the agent is allowed to stay unmoved or move to the adjacent states. Here $s_1$ is the 0-reward state to be avoided, $s_2$ and $s_4$ are higher-reward states subject to risk after perturbation, and $s_3$ is the lower-reward safe state. Suppose that all transition probabilities are subject to uncertainty $\delta$.

\begin{figure}[ht]
  \centering
  \begin{subfigure}[b]{0.36\linewidth}
    \centering
      \tikzset{snode/.style = {draw=black, shape=circle, line width=1.0pt, inner sep=0pt, minimum width=16pt}}
  \tikzset{apath/.style = {draw=black, line width=0.5pt, ->, >={Stealth}}}
  \begin{tikzpicture}
    \node[snode, fill=red!30!white] (s1) at (0pt, 20pt) {$s_1$};
    \node[snode, fill=yellow!30!white] (s2) at (20pt, 0pt) {$s_2$};
    \node[snode, fill=green!30!white] (s3) at (0pt, -20pt) {$s_3$};
    \node[snode, fill=yellow!30!white] (s4) at (-20pt, 0pt) {$s_4$};

    \draw[apath] (s1) edge[bend left] (s2);
    \draw[apath] (s2) edge[bend left] (s3);
    \draw[apath] (s3) edge[bend left] (s4);
    \draw[apath] (s4) edge[bend left] (s1);

    \draw[apath] (s2) edge[bend left] (s1);
    \draw[apath] (s3) edge[bend left] (s2);
    \draw[apath] (s4) edge[bend left] (s3);
    \draw[apath] (s1) edge[bend left] (s4);
    
    \draw[apath] (s1) edge[loop above] (s1);
    \draw[apath] (s2) edge[loop right] (s2);
    \draw[apath] (s3) edge[loop below] (s3);
    \draw[apath] (s4) edge[loop left] (s4);
  \end{tikzpicture}
    \caption{Setup.}\label{fig:6-simulation_MDP}
  \end{subfigure}
  \begin{subfigure}[b]{0.6\linewidth}
    \centering
    \includegraphics[width=0.56\linewidth]{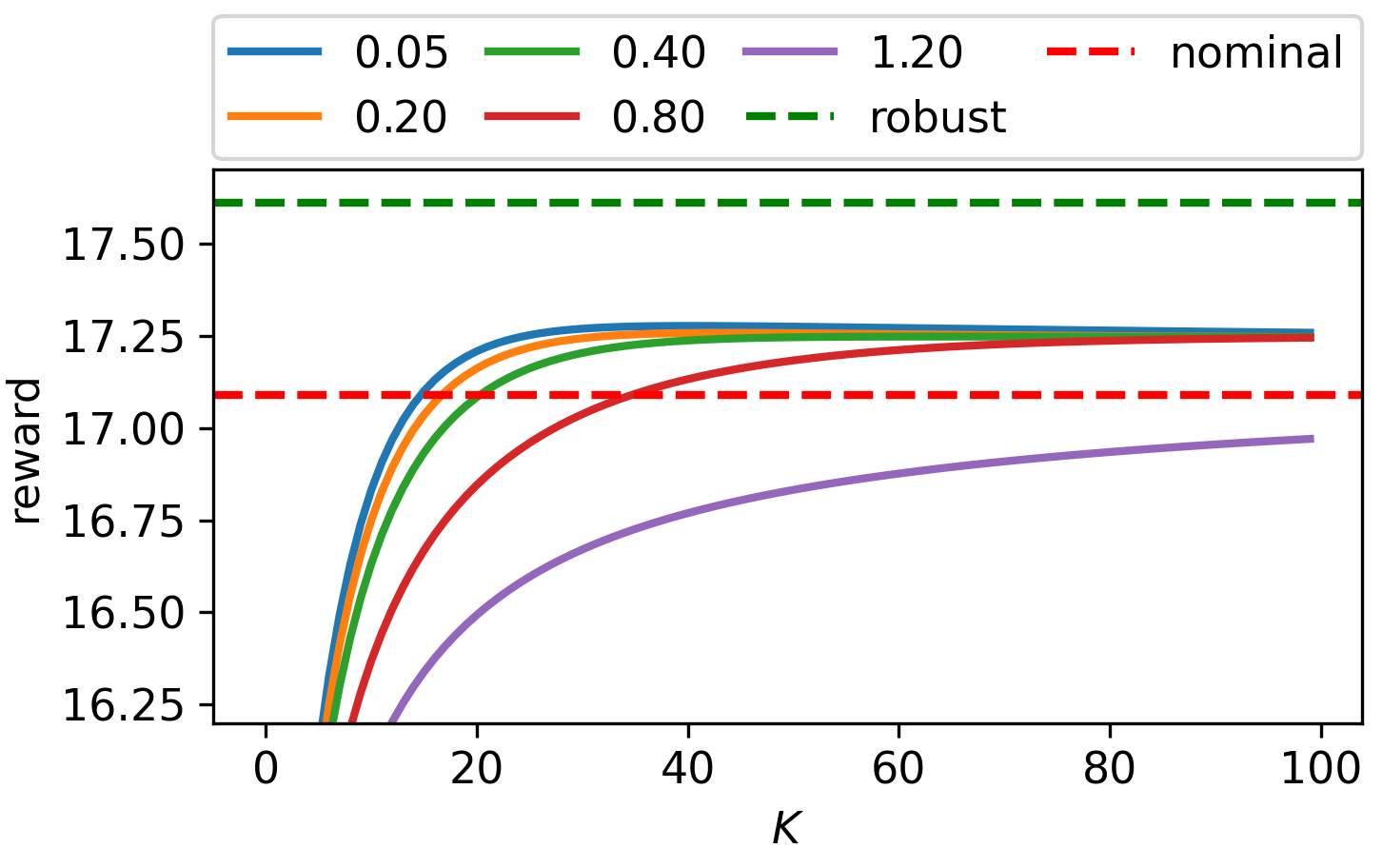}
    \caption{Policy improvement over time.}\label{fig:6-simulation_results}
  \end{subfigure}
  \vspace{-2mm}
  \caption{Numerical simulations in a toy model. }
\end{figure}
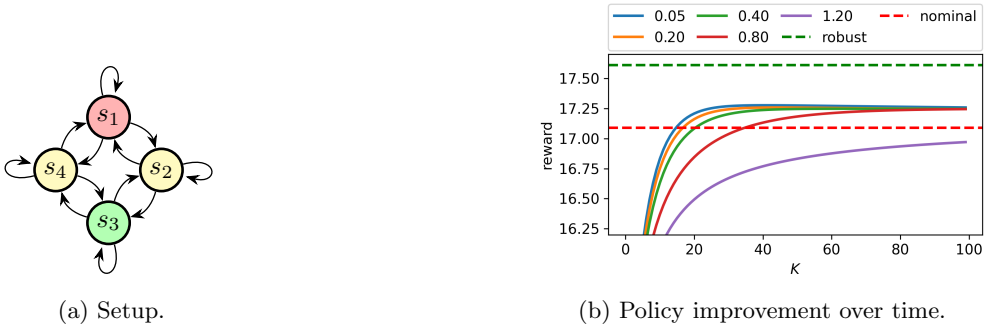

The R\tsup{2}PG algorithm is run with different perturbation radii, and the policies obtained in all the episodes are evaluated by the \textit{minimum} cumulative reward evaluated in a few perturbed MDPs. Simulation results for $R_{\eta,h} = 0.01$ and $R_{\xi,h} \in \set{0.05,0.2,0.4,0.8,1.2}$ are plotted in \Cref{fig:6-simulation_results}. It can be observed that policies tend to converge in all executions, and a larger perturbation radius generally leads to more conservative behavior. This phenomenon is largely expected in that, as perturbation radius increases, the misspecification error induced by the worst-case pseudo-MDP also increases, which leads to an intrinsically pessimistic estimation of policy values. However, the output policies still perform better than the nominal optimal policy when the MDP is appropriately perturbed, highlighting again the need for robustness in environments with uncertainty.

Analyzing the output policies in details, we shall further find that, as time elapses, all output policies gradually lean towards the safe state by increasing the transition probabilities to it. However, since the D\tsup{2}PG algorithm is designed to optimize over the average performance for policy evaluation, it is also reasonable that it does not fully converge to a policy that yields optimal worst-case performance.

Details of the simulation can be found in \Cref{sec:appx-D-simulation}.
  \vspace{-6pt}
\section{Conclusion}
\vspace{-3pt}

In this paper, we propose a novel robustness concept based on $(\xi, \eta)$-rectangularity, which achieves efficient dual perturbation robustness in low-rank MDPs. The new robustness concept features computational efficiency, scalability and compatability with low-rank representation structure. Based on the new robustness concept, we design an algorithm (D\tsup{2}PG) to solve the proposed robust low-rank MDP that provably converges to the optimal robust policy with bounded suboptimality gap.

Future work includes designing algorithms that solve for the robust policy in an asymptotically accurate and/or more computationally efficient way, incorporating sample-based methods to estimate the nominal MDP and use the estimated MDP to generate robust policies, and further, discovering other robustness concepts that are compatible with different low-rankness concepts for more scalable robust RL.


  \bibliography{biblio}
  \bibliographystyle{unsrtnat}

  \newpage
  \appendix
  
  \section{Rationale of the Proposed Low-rank Robustness with \titlemath{(\bxi, \bm{\eta})}-Rectangularity}

In this section we present a series of properties and examples to illustrate the rationale of the low-rank robustness concept proposed in this paper.

\subsection{Transformation from \titlemath{(\bphi,\bmu,\bnu)}- to \titlemath{(\bxi,\bm{\eta})}-Rectangular Ambiguity Sets}\label{sec:appx-A-1-transformation}

We start by specifying how to transform between $(\phi,\mu,\nu)$-rectangular ambiguity sets (required by the proposed low-rank robustness concept) and $(\xi,\eta)$-rectangular ambiguity sets (required by the effective robustness concept). Specifying the transformation may also help to promote the understanding of the new robustness concept.

Recall the definition of  $(\phi,\mu,\nu)$-rectangular ambiguity sets in \eqref{eq:3-standard_robust_MDP}. For the sake of transformation, we would like to rewrite the optimization problem \eqref{eq:3-robust_update} in terms of perturbations; i.e., define
\begin{subequations}
\begin{align}
  \bdelta_{\phi, h}(s,a) &:= \bphi_h(s,a)-\bphi_h^{\circ}(s,a),~ \forall (s,a) \in \S \times \A;\\
  \bdelta_{\mu, h}(s') &:= \bmu_h(s')-\bmu_h^{\circ}(s'),~ \forall s' \in \S;\\
  \bdelta_{\nu, h} &:= \bnu_h -\bnu_h^{\circ},
\end{align}
\end{subequations}
so that for any $(s,a) \in \S \times \A$, \eqref{eq:3-robust_update} shall be equivalently rewritten as
\begin{equation}\label{eq:3-robust_update_delta}
  \min_{\bdelta_{\phi,h},\bdelta_{\mu,h}, \bdelta_{\nu,h}} \angl*{\bphi_h^{\circ}(s,a) + \bdelta_{\phi,h}(s,a), \bnu_h^{\circ} + \bdelta_{\nu,h} + \sum_{s'} \Vrob^{\pi}_{h+1}(s') \prn[\big]{\bmu_h^\circ(s') + \bdelta_{\mu,h}(s')}}.
\end{equation}
It is evident that $\bxi_h$ is simply a collection of perturbation terms in the second component of the inner product, namely
\begin{equation}
  \bxi_h := \bdelta_{\nu,h} + \sum_{s'} \Vrob^{\pi}_{h+1}(s') \bdelta_{\mu,h}(s'),~ \forall h \in [H].
\end{equation}
Meanwhile, since in the optimization problem \eqref{eq:3-effective_robust_update} that formulates the effective robustness concept, we only care about the weighted average of \eqref{eq:3-robust_update_delta} over $d^{\pi}_h$ for some policy $\pi$, we shall define
\begin{equation}
  \bm{\eta}_h := \E[(s,a) \sim d^{\pi}_h]{\bdelta_{\phi,h}(s,a)},~ \forall h \in [H].
\end{equation}
In this way, we have recovered the form of \eqref{eq:3-effective_robust_update}.

To determine the range of $\bm{\eta}_h$, note that by linearity of expectation, we always have
\begin{equation}\label{eq:A-1-linearity_of_eta}
  \E[s,a]{\bphi_h^{\circ}(s,a) + \bdelta_{\phi,h}(s,a)}
  = \E[s,a]{\bphi_h^{\circ}(s,a)} + \bm{\eta}_h
  = \E[s,a]{\bphi_h^{\circ}(s,a) + \bm{\eta}_h}.
\end{equation}
Equation \eqref{eq:A-1-linearity_of_eta} can be interpreted in two directions: on the one hand, given perturbations $\bdelta_{\phi,h}(\cdot,\cdot)$ such that $\norm{\bdelta_{\phi,h}(s,a)} \leq \r_{\phi,h}$, we always have $\norm{\bm{\eta}_h} \leq \r_{\phi,h}$; on the other hand, given the averaged perturbation $\bm{\eta}_h$, we shall simply regard each individual perturbation $\bdelta_{\phi,h}(s,a)$ identically as $\bm{\eta}_h$ to recover \textit{one} possible realization of the averaged perturbation. Therefore, if we assume $\bphi_h(\cdot,\cdot)$ are perturbed within radius $\r_{\phi,h}$, $\bm{\eta}_h$ can also be regarded as perturbed within radius $\r_{\phi,h}$.

We would like to point out that a complicated set of implicit constraints should be appended to the $(\xi,\eta)$-rectangular ambiguity set to keep the transformation equivalent, which contradicts with the objective of reducing computational complexity for introducing the $(\xi,\eta)$-ambiguity set. Therefore, we shall always relax the constraints for the $(\xi, \eta)$-ambiguity set after transformation, and use an upper bound for $\r_{\xi,h}$ and $\r_{\eta,h}$, so that all the MDPs represented by the original $(\phi,\mu,\nu)$-ambiguity set are included in the new $(\xi,\eta)$-ambiguity set. For example, when we are given a $V$-function to be updated, we shall regard the $(\xi,\eta)$-rectangular ambiguity set after transformation as specified by the radii
\begin{equation}\label{eq:A-1-ambiguity_radius_transform}
  \r_{\xi,h} = \r_{\nu,h} + \sum_{s'} |V_h(s')| \r_{\mu,h},\quad
  \r_{\eta,h} = \r_{\phi,h}.
\end{equation}
We may further replace each $|V_h(s')|$ with a trivial upper bound $H-h$. The takeaway message here is that we can always select appropriate $\r_{\xi,h}$ and $\r_{\eta,h}$ so that the transformed $(\xi,\eta)$-ambiguity set is a relaxation of the original ambiguity set.

\begin{remark}
  It is evident that some information gets lost when we collect the perturbation terms into $\bxi_h$, in the sense that a ball constraint for $\bxi_h$ might be equivalent to a continuum of different constraints for $(\bnu_h, \bmu_h)$, making it impossible to uniquely determine the inverse transformation.
\end{remark}

\subsection{The Relationship Between Nominal, Standard Robust and the Low-rank Robust Values}\label{sec:appx:A-2-value_relation}

We proceed to study the relationship between the three different Bellman operators and the corresponding value functions. Throughout this section, we will assume that all MDPs specified by the $(\phi,\mu,\nu)$-ambiguity set $\M$ (for the standard robust Bellman operator) are included in the $(\xi,\eta)$-ambiguity set $\widehat{\M}$ (for the effective robust Bellman operator), which is feasible due to the discussion in the previous section (\Cref{sec:appx-A-1-transformation}).

For the sake of clarity, we collect their definitions (with respect to policy $\pi$) here for exposition:
\begin{itemize}
  \item The nominal $V$- and $Q$-functions, i.e. $\Vnom^{\pi}_h(\cdot)$ and $\Qnom^{\pi}_h(\cdot,\cdot)$, are defined in \eqref{eq:2-nominal_value}. The nominal Bellman operator is specified by the Bellman equation (let $\P^{\circ}_h(s' \mid s,a) = \angl{\bphi^{\circ}_h(s,a), \bmu^{\circ}_h(s')}$ and $r^{\circ}_h(s,a) = \angl{\bphi^{\circ}_h(s,a), \bnu^{\circ}_h}$)
  \begin{equation}
    \Qnom^{\pi}_h(s,a) = \E[s' \sim \P^{\circ}_h(\cdot | s, a)]{r^{\circ}_h(s, a) + \Vnom^{\pi}_{h+1}(s')} =: [\Bnom_h \Vnom^{\pi}_{h+1}](s,a).
  \end{equation}

  \item The standard robust $V$-function $\Vrob^{\pi}_h(\cdot)$ is defined in \eqref{eq:3-standard_robust_value}. The corresponding $Q$-function $\Qrob^{\pi}_h(\cdot,\cdot)$ and the standard robust Bellman operator are specified by the robust Bellman equation
  \begin{equation}
    \Qrob^{\pi}_h(s,a) := \min_{\begin{subarray}{c} (\bphi_{1:H}, \bmu_{1:H}, \bnu_{1:H}) \in \M,\\ \mathbb{P}_h = \angl{\bphi_h,\bmu_h}, r_h = \angl{\bphi_h,\bnu_h} \end{subarray}} \E[s' \sim \P_h(\cdot | s, a)]{r_h(s, a) + \Vrob^{\pi}_{h+1}(s')} =: [\Brob_h \Vrob^{\pi}_{h+1}](s,a).
  \end{equation}

  \item The proposed low-rank robust Bellman operator is defined via \eqref{eq:3-effective_robust_update}, \eqref{eq:3-effective_ambiguity_set} and \eqref{eq:3-effective_robust_operator}, and the corresponding robust $V$-function is recursively defined in \eqref{eq:3-effective_robust_value} with $\Vrlx^{\pi}_{H+1}(\cdot) \equiv 0$. The robust $Q$-function is further defined as $\Qrlx^{\pi}_h(s,a) := [\Brlx^{\pi}_h \Vrlx^{\pi}_{h+1}](s,a)$.
\end{itemize}
It is worth pointing out that $\Vnom^{\pi}_h(s) = \angl{\Qnom^{\pi}_h(s,\cdot), \pi_h(\cdot | s)}$, $\Vrob^{\pi}_h(s) = \angl{\Qrob^{\pi}_h(s,\cdot), \pi_h(\cdot | s)}$ and $\Vrlx^{\pi}_h(s) = \angl{\Qrlx^{\pi}_h(s,\cdot), \pi_h(\cdot | s)}$.

\paragraph{Ordinal Relation.} It is intuitive that, for any policy $\pi$, its robust values should be no greater than its nominal value, as the former represents certain kinds of worst-case evaluations within ambiguity sets that include the nominal model used in the latter evaluation; further, its ``relaxed'' robust value should also be no greater than its standard robust value (assuming the transform specified in \eqref{eq:A-1-ambiguity_radius_transform}), since the former ambiguity set is a superset of the latter.

To formally prove the first part of \Cref{thm:3-value_relation}, we start by showing a few useful lemmas capturing the above intuition.

\begin{lemma}\label{thm:A-2-update_ordering}
  For any step $h \in [H]$, any $V$-function $V: \S \to \R$ and any policy $\pi$, we have
  \begin{equation}
    [\Brlx^{\pi}_h V](s,a) \leq [\Brob_h V](s,a) \leq [\Bnom_h V](s,a),~ \forall (s,a) \in \S \times \A.
  \end{equation}
\end{lemma}

\begin{proof}
  This lemma is basically formalizing the intuition stated above. Note that $(\bphi^{\circ}, \bmu^{\circ}, \bnu^{\circ}) \in \M$, so we have
  \begin{subequations}
  \begin{align}
    [\Brob_h V](s,a)
    &= \min_{\begin{subarray}{c} (\bphi_{1:H}, \bmu_{1:H}, \bnu_{1:H}) \in \M,\\ \mathbb{P}_h = \angl{\bphi_h,\bmu_h}, r_h = \angl{\bphi_h,\bnu_h} \end{subarray}} \E[s' \sim \P_h(\cdot | s, a)]{r_h(s, a) + V(s')} \\
    &\leq \E[s' \sim \P^{\circ}_h(\cdot | s, a)]{r^{\circ}_h(s, a) + V(s')} \\
    &= [\Bnom_h V](s,a).
  \end{align}
  \end{subequations}
  Further, let $(\bphi^*, \bmu^*, \bnu^*) \in \M$ be the optimal solutions to \eqref{eq:3-robust_update} that attain minimum with respect to $V$, and define
  \begin{equation}
    \bxi^*_h := (\bnu^*_h - \bnu^{\circ}_h) + \sum_{s'} V(s') \prn[\big]{\bmu^*_h(s') - \bmu^{\circ}_h(s')},\quad
    \bm{\eta}_h^* = \E[(s,a) \sim d^{\pi}_h]{\bphi^*_h(s,a) - \bphi^{\circ}_h(s,a)}.
  \end{equation}
  By assumption we have $(\bxi^*_h, \bm{\eta}^*_h) \in \widehat{\M}$, and thus shall derive
  \begin{subequations}
  \begin{align}
    [\Brlx^{\pi}_h V](s,a)
    &= \min_{(\bxi_h, \bm{\eta}_h) \in \widehat{\M}_h} \angl*{\bphi^{\circ}_h(s,a) + \bxi_h, \bomega^{\circ}_V + \bm{\eta}_h} \\
    &\leq \angl*{\bphi^{\circ}_h(s,a) + \bxi^*_h, \bomega^{\circ}_V + \bm{\eta}^*_h} \\
    &= [\Brob_h V](s,a),
  \end{align}
  \end{subequations}
  where $\bomega^{\circ}_V := \bnu^{\circ}_h + \sum_{s'} V(s') \bmu^{\circ}_h(s')$.
\end{proof}

\begin{lemma}\label{thm:A-2-robust_ordering}
  For any $V$-functions $V, V': \S \to \R$ such that $V(s) \leq V'(s),~ \forall s \in \S$, we have
  \begin{equation}
    [\Brob_h V](s,a) \leq [\Brob_h V'](s,a),~ \forall (s,a) \in \S \times \A.
  \end{equation}
\end{lemma}

\begin{proof}
  Let $(\bphi^*, \bmu^*, \bnu^*) \in \M$ be the optimal solution of \eqref{eq:3-robust_update} that attains minimum with respect to $V'$. Then we have
  \begin{subequations}
  \begin{align}
    [\Brob_h V'](s,a)
    &= \angl[\Big]{\bphi^*_h(s,a), \bnu^*_h + \sum_{s'} V'(s') \bmu^*_h(s')} \\
    &= \angl*{\bphi^*_h(s,a), \bnu^*_h} + \sum_{s'} V'(s') \angl*{\bphi^*_h(s,a), \bmu^*_h(s')} \\
    &\geq \angl*{\bphi^*_h(s,a), \bnu^*_h} + \sum_{s'} V(s') \angl*{\bphi^*_h(s,a), \bmu^*_h(s')} \\
    &= \angl[\Big]{\bphi^*_h(s,a), \bnu^*_h + \sum_{s'} V(s') \bmu^*_h(s')} \\
    &\geq \min_{\bphi_h, \bmu_h, \bnu_h} \angl[\Big]{\bphi_h(s,a), \bnu_h + \sum_{s'} V(s') \bmu_h(s')} \\
    &= [\Brob_h V](s,a).
  \end{align}
  \end{subequations}
  Here we use the fact $\angl*{\bphi^*_h(s,a), \bmu^*_h(s')} \geq 0$, since $(\bphi^*, \bmu^*, \bnu^*)$ represents a linear MDP.
\end{proof}

Now we shall proceed to show the first part of \Cref{thm:3-value_relation}.

\begin{proof}[Proof of \Cref{thm:3-value_relation} (first part)]
  This is a proof by mathematical induction on $h$. For the base case $h = H+1$, the inequality is trivial, since by definition we have $\Qrlx^{\pi}_{H+1}(\cdot,\cdot) = \Qrob^{\pi}_{H+1}(\cdot,\cdot) = \Qnom^{\pi}_{H+1}(\cdot,\cdot) = 0$ and $\Vrlx^{\pi}_{H+1}(\cdot) = \Vrob^{\pi}_{H+1}(\cdot) = \Vnom^{\pi}_{H+1}(\cdot) = 0$.

  For the induction step, suppose we have already shown $\Vrlx^{\pi}_{h+1}(s) \leq \Vrob^{\pi}_{h+1}(s) \leq \Vnom^{\pi}_{h+1}(s)$. Then we have
  \begin{equation}
    \underbrace{[\Brlx^{\pi}_h \Vrlx^{\pi}_{h+1}](s,a)}_{\Qrlx^{\pi}_h(s,a)}
    \leq [\Brob_h \Vrlx^{\pi}_{h+1}](s,a)
    \leq \underbrace{[\Brob_h \Vrob^{\pi}_{h+1}](s,a)}_{\Qrob^{\pi}_h(s,a)}
    \leq [\Brob_h \Vnom^{\pi}_{h+1}](s,a)
    \leq \underbrace{[\Bnom_h \Vnom^{\pi}_{h+1}](s,a)}_{\Qnom^{\pi}_h(s,a)},
  \end{equation}
  where we apply \Cref{thm:A-2-update_ordering} and \Cref{thm:A-2-robust_ordering}. Further, by taking inner product with the policy $\pi_h$, we have
  \begin{equation}
    \underbrace{\angl[\big]{\pi_h(\cdot | s), \Qrlx^{\pi}_h(s,\cdot)}}_{\Vrlx^{\pi}_h(s)}
    \leq \underbrace{\angl[\big]{\pi_h(\cdot | s), \Qrob^{\pi}_h(s,\cdot)}}_{\Vrob^{\pi}_h(s)}
    \leq \underbrace{\angl[\big]{\pi_h(\cdot | s), \Qnom^{\pi}_h(s,\cdot)}}_{\Vnom^{\pi}_h(s)},
  \end{equation}
  This completes the proof.
\end{proof}

\paragraph{Bounded Gap.} We proceed to show that the gap between the nominal, standard robust and low-rank robust value functions are upper bounded by a function of the perturbation radii. For this purpose, we shall first shown a few technical lemmas.

\begin{lemma}\label{thm:A-2-diff_nominal_effective}
  For any $V$-function $V: \S \to \R$ and any policy $\pi$, we have
  \begin{equation}
    [\Bnom_h V](s,a) - [\Brlx^{\pi}_h V](s,a) \leq 2 \r_{\eta,h} \sqrt{d} + (1+\r_{\eta,h}) \r_{\xi,h},~ \forall (s,a) \in \S \times \A.
  \end{equation}
\end{lemma}

\begin{proof}
  Let $(\bxi^*, \bm{\eta}^*)$ be the optimal solutions of \eqref{eq:3-effective_robust_update} that attain minimum with respect to $V$ and $\Brlx^{\pi}_h$. Then we have
  \begin{subequations}
  \begin{align}
    [\Bnom_h V](s,a) - [\Brlx^{\pi}_h V](s,a)
    &= \angl{\bphi^{\circ}_h(s,a) + \bm{\eta}^*, \bomega^{\circ}_h + \bxi^*} - \angl{\bphi^{\circ}_h(s,a), \bomega^{\circ}_h} \\
    &= \angl{\bm{\eta}^*, \bomega^{\circ}_h} + \angl{\bxi^*, \bphi^{\circ}_h + \bm{\eta}^*} \\
    &\leq \norm{\bm{\eta}^*} \norm{\bomega^{\circ}_h} + \norm{\bxi^*} \prn[\big]{\norm{\bphi^{\circ}_h} + \norm{\bm{\eta}^*}} \\
    &\leq 2 \r_{\eta,h} \sqrt{d} + (1+\r_{\eta,h}) \r_{\xi,h},
  \end{align}
  \end{subequations}
  where we use the fact that $\norm{\bomega^{\circ}_h} = \norm{\bnu^{\circ}_h + \sum_{s'} V(s') \bmu^{\circ}_h(s')} \leq 2\sqrt{d}$. This completes the proof.
\end{proof}

\begin{lemma}\label{thm:A-2-nominal_Lipschitz}
  For any $V$-function $V: \S \to \R$, we have $\norm{\Bnom_h V - \Bnom_h V'}_{\infty} \leq \norm{V - V'}_{\infty}$.
\end{lemma}

\begin{proof}
  By definition we have
  \begin{equation}
    [\Bnom_h V](s,a) - [\Bnom_h V'](s,a)
    = \E[s' \sim \P^{\circ}_h(\cdot | s, a)]{V(s') - V'(s')}
    \leq \norm{V - V'}_{\infty}.
  \end{equation}
  Taking maximum over the left-hand side gives the desired inequality.
\end{proof}

\begin{proof}[Proof of \Cref{thm:3-value_relation} (second part)]
  We proceed by mathematical induction on $h$ to show that
  \begin{equation}
    \norm{\Vrob^{\pi}_h - \Vrlx^{\pi}_h}_{\infty}
    \leq \norm{\Vnom^{\pi}_h - \Vrlx^{\pi}_h}_{\infty} 
    \leq \sum_{\tau = h}^{H} \prn[\big]{2 \r_{\eta,\tau} \sqrt{d} + (1+\r_{\eta,\tau}) \r_{\xi,\tau}},~ \forall h \in [H+1].
  \end{equation}
  The base case $h = H+1$ trivially holds, since by definition we have $\Vrlx^{\pi}_{H+1}(\cdot) = \Vrob^{\pi}_{H+1}(\cdot) = \Vnom^{\pi}_{H+1}(\cdot) = 0$. For the induction step at $h$, note that
  \begin{subequations}
  \begin{align}
    \norm{\Vrob^{\pi}_h - \Vrlx^{\pi}_h}_{\infty}
    &\leq \norm{\Vnom^{\pi}_h - \Vrlx^{\pi}_h}_{\infty} \\
    &\leq \norm{\Bnom_h \Vnom^{\pi}_{h+1} - \Brlx^{\pi}_h \Vrlx^{\pi}_{h+1}}_{\infty} \\
    &\leq \norm{\Bnom_h \Vnom^{\pi}_{h+1} - \Bnom_h \Vrlx^{\pi}_{h+1}}_{\infty} + \norm{\Bnom_h \Vrlx^{\pi}_{h+1} - \Brlx^{\pi}_h \Vrlx^{\pi}_{h+1}}_{\infty} \\
    &\leq \norm{\Vnom^{\pi}_{h+1} - \Vrlx^{\pi}_{h+1}}_{\infty} + 2 \r_{\eta,h} \sqrt{d} + (1+\r_{\eta,h}) \r_{\xi,h} \\
    &\leq \sum_{\tau = h}^{H} \prn[\big]{2 \r_{\eta,\tau} \sqrt{d} + (1+\r_{\eta,\tau}) \r_{\xi,\tau}},
  \end{align}
  \end{subequations}
  where we apply \Cref{thm:A-2-diff_nominal_effective} and \Cref{thm:A-2-nominal_Lipschitz} for the second last inequality, and the induction hypothesis for the last inequality. Finally, we apply \Cref{thm:A-2-update_ordering} to derive the desired inequality.
\end{proof}

\paragraph{The String Guessing Example.} The above theoretical results fail to answer the question whether the effective robust value is a good approximation of the robust value, and it may seem suspicious that the gap is only upper bounded by $\Theta(H)$. To justify why we cannot expect anything better in the worst case, we include the following example for illustration.

Consider a string guessing game with the answer set to be an $m$-bit binary string. Without loss of generality, let  $11 \cdots 1$ be the answer (otherwise we shall rename the two actions on any bit of 0). There are two actions, i.e. $\A = \set{a_0,a_1}$. The game proceeds in a bit-wise manner, and the transient states $s_{1:m}$ are used to record the progress. There are two absorptive states: $s_-$ for an error on any bit, which yields a reward of 0 for each of the remaining steps; $s_+$ for success on all bits, which yields a reward of 1 for each of the remaining steps. Therefore, the state space is $\S = \brac{s_-,s_+,s_1,s_2,\ldots,s_m}$ with a deterministic initial state $s_1$. The MDP is illustrated in \Cref{fig:A-2-string_guessing} below, where transitions are deterministic (as indicated by the arrows), and all rewards are 0 except for the self-loop at $s_+$.

\begin{figure}[H]
  \centering
    \tikzset{snode/.style = {draw=black, shape=circle, line width=1.0pt, inner sep=0pt, minimum width=16pt}}
  \begin{tikzpicture}[%
    every path/.style = {draw=black, line width=0.5pt, ->, >={Stealth}}%
  ]
    \node[snode] (s0) at (80pt, 50pt) {$s_-$};
    \node[snode] (s1) at (0pt, 0pt) {$s_1$};
    \node[snode] (s2) at (40pt, 0pt) {$s_2$};
    \node[snode] (s3) at (80pt, 0pt) {$s_3$};
    \node (s4) at (120pt, 0pt) {$\cdots$};
    \node[snode] (sm) at (160pt, 0pt) {$s_m$};
    \node[snode] (sp) at (200pt, 0pt) {$s_+$};

    \draw (s1) edge node[above]{$a_1$} (s2);
    \draw (s2) edge node[above]{$a_1$} (s3);
    \draw (s3) edge node[above]{$a_1$} (s4);
    \draw (s4) edge node[above]{$a_1$} (sm);
    \draw (sm) edge node[above]{$a_1$} (sp);
    \draw (sp) edge[loop above] node[above]{\scriptsize\color{blue} $r=1$} (sp);

    \draw (s1) edge node[left]{$a_0$} (s0);
    \draw (s2) edge node[right]{$a_0$} (s0);
    \draw (s3) edge node[right]{$a_0$} (s0);
    \draw (sm) edge node[right]{$a_0$} (s0);
    \draw (s0) edge[loop above] node[left]{} (s0);
  \end{tikzpicture}
  \caption{MDP diagram for the string guessing game.}\label{fig:A-2-string_guessing}
\end{figure}
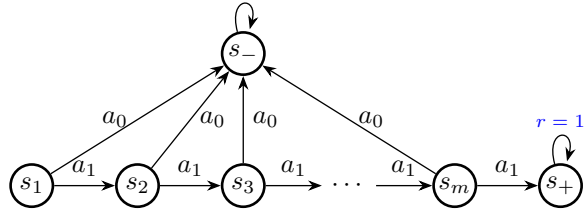

We proceed to define the low-rank representation for this MDP. Set the feature dimension to be $d = 4$, and the feature vectors for each state-action pair at time step $h \in [H]$ are defined as
\begin{equation}
  \bphi^{\circ}_h(s_-, \cdot) = \begin{bmatrix}
    1 \\ 0 \\ 0 \\ 0
  \end{bmatrix};\quad
  \bphi^{\circ}_h(s_i, a_0) = \begin{bmatrix}
    0 \\ 1 \\ 0 \\ 0
  \end{bmatrix},~
  \bphi^{\circ}_h(s_i, a_1) = \begin{bmatrix}
    0 \\ 0 \\ 1 \\ 0
  \end{bmatrix} (\forall i \in [m]);\quad
  \bphi^{\circ}_h(s_+, \cdot) = \begin{bmatrix}
    0 \\ 0 \\ 0 \\ 1
  \end{bmatrix},
\end{equation}
with corresponding factors (for the sake of convenience, unspecified factors are set to $\bm{0}$ by default)
\begin{equation}
  \bnu^{\circ}_h = \begin{bmatrix}
    0 \\ 0 \\ 0 \\ 1
  \end{bmatrix};\quad
  \bmu^{\circ}_h(s_-) = \begin{bmatrix}
    1 \\ 1 \\ 0 \\ 0
  \end{bmatrix},~
  \bmu^{\circ}_h(s_{h+1}) = \begin{bmatrix}
    0 \\ 0 \\ 1 \\ 0
  \end{bmatrix} (\forall h \leq m),~
  \bmu^{\circ}_h(s_+) = \begin{bmatrix}
    0 \\ 0 \\ 0 \\ 1
  \end{bmatrix} (\forall h > m),
\end{equation}
where we regard $s_{m+1}$ as $s_+$ for simplicity.

To formulate robust MDPs around the nominal model, suppose that the transition probabilities are subject to uncertainty, so that each transition probability may differ from the nominal value by at most $\delta$. Specifically,
\begin{itemize}
  \item for the standard robust MDP, set $\r_{\phi,h} = \r_{\nu,h} = 0$ and $\r_{\mu,h} = \delta$ for the $(\phi,\mu,\nu)$-rectangular ambiguity set $\M$;
  \item for the low-rank robust MDP, set $\r_{\xi,h} = (H-h)\delta$ and $\r_{\eta,h} = 0$ for the $(\xi,\eta)$-rectangular ambiguity set $\widehat{\M}$.
\end{itemize}
Based on the discussion in \Cref{sec:appx-A-1-transformation}, we know that the low-rank robust ambiguity set $\widehat{\M}$ is a relaxation of the standard robust ambiguity set $\M$.

Now we consider a policy $\pi$ that always takes action $a_1$. The nominal, standard robust and low-rank robust values for policy $\pi$ shall be calculated as follows:
\begin{itemize}
  \item \textit{The nominal value.} It is clear that, in the nominal model, the cumulative reward obtained by policy $\pi$ is $H-m$, which comes from the last $H-m$ steps at $s_+$. Hence $\Vnom^{\pi}_1(s_1) = H-m$.

  \item \textit{The standard robust value.} Among all possible perturbations, the worst case happens when there is a $\delta$ probability for $a_1$ to lead to $s_-$ at any $s_i$, so that the worst-case factors are given by
  \begin{equation}
    \bmu_h(s_-) = \begin{bmatrix}
      1 \\ 1 \\ \delta \\ 0
    \end{bmatrix},~
    \bmu_h(s_{h+1}) = \begin{bmatrix}
      0 \\ 0 \\ 1-\delta \\ 0
    \end{bmatrix} (\forall h \leq m).
  \end{equation}
  It is straight-forward to verify that it remains a valid MDP after perturbation. In this case the standard robust value is given by $\Vrob^{\pi}_1(s_1) = (1-\delta)^m (H-m)$, which is exponentially far from the nominal value as $m$ increases.

  \item \textit{The low-rank robust value.} Note that, in \eqref{eq:3-effective_robust_update}, since $\bphi^{\circ}_h(\cdot, \cdot)$ is left unperturbed, while the nominal state distribution $\rho^{\pi}_h$ is a point mass at $s_i$ (for $h \leq m$) or $s_+$ (for $h > m$). Therefore, the optimal $\xi^*_h$ is always given by $- \r_{\xi,h} \frac{\bphi^{\circ}(s_i, a_1)}{\norm{\bphi^{\circ}(s_i, a_1)}}$, and thus an additional $(H-h)\delta$ should be subtracted from the lowr-tank robust Bellman update at time step $h$, as compared to the nominal Bellman update. Therefore, the low-rank robust value is given by $\Vrlx^{\pi}_1(s_1) = (H - m) - \sum_{h=1}^{m} (H-h)\delta$.
\end{itemize}
With these values in hand, we first verify their ordinal relations, which follows from Bernoulli's inequality as
\begin{equation}
  \Vnom^{\pi}_1(s_1)
  > \Vrob^{\pi}_1(s_1)
  = (1-\delta)^m (H-m)
  \geq (1-m\delta) (H-m)
  = (H - m) - \sum_{h=1}^{m} (H-m)\delta
  > \Vrlx^{\pi}_1(s_1).
\end{equation}
Meanwhile, it can be shown by Taylor expansion that
\begin{align}
  \Vrob^{\pi}_1(s_1) - \Vrlx^{\pi}_1(s_1)
  &= \sum_{h=1}^{m} (H-h)\delta - (H-m) \cdot m\delta + \o(\delta) 
  = \frac{m(m+1)}{2} \delta + \o(\delta), \\
  \Vnom^{\pi}_1(s_1) - \Vrob^{\pi}_1(s_1) &= (H-m)m \delta + \o(\delta).
\end{align}
Therefore, with sufficiently small $\delta$, $\Vrob^{\pi}_1(s_1)$ would be much closer to $\Vrlx^{\pi}_1(s_1)$ when $H \gg m$, but become much closer to $\Vnom^{\pi}_1(s_1)$ when $H = \Theta(m)$, which implies that in general we cannot expect anything better than \Cref{thm:3-value_relation}.

\subsection{Robustness Induced by Low-rank Robust MDPs}\label{sec:appx-A-3-robustness_example}

Although we have illustrated the rationale behind the proposed robust low-rank MDPs, it is still unclear whether solving the low-rank robust planning problem \Cref{eq:3-effective_robust_policy} actually leads to a policy that displays certain level of robustness. For this purpose, we include the following example to compare the optimal effective robust policy under different perturbation radii to demonstrate its robust behavior.

Consider the following ``gamble-or-guarantee'' game where the agent is required to enter one of the two branches: a no-risk ``guarantee'' branch (taking action $a_0$) including a single absorptive state $s_{\alpha}$ to receive a constant reward $\alpha$ from then on, and a risky ``gamble'' branch (taking action $a_1$) that includes a potentially transient state $s_1$ to receive rewards of $1$, at the risk of permanently falling into the 0-reward absorption state $s_0$. The MDP is illustrated in \Cref{fig:A-3-risk_seeking} below.

\begin{figure}[ht]
  \centering
    \tikzset{snode/.style = {draw=black, shape=circle, line width=1.0pt, inner sep=0pt, minimum width=16pt}}
  \tikzset{apath/.style = {draw=black, line width=0.5pt, ->, >={Stealth}}}
  \begin{tikzpicture}
    \node[snode] (si) at (0pt, 40pt) {$s_+$};
    \node[snode] (sa) at (-20pt, 0pt) {$s_{\alpha}$};
    \node[snode] (s1) at (20pt, 0pt) {$s_1$};
    \node[snode] (s0) at (60pt, 0pt) {$s_0$};

    \draw[apath] (si) edge node[left]{$a_0$} (sa);
    \draw[apath] (si) edge node[right]{$a_1$} (s1);
    \draw[apath] (s1) edge node[above=-1pt]{\scriptsize $p$} (s0);
    
    \draw[apath] (sa) edge[loop below] node[below]{\scriptsize\color{blue} $r=\alpha$} (sa);
    \draw[apath] (s1) edge[loop below] node[left=4pt,yshift=3pt]{\scriptsize $1-p$} node[below]{\scriptsize\color{blue} $r=1$} (s1);
    \draw[apath] (s0) edge[loop below] node[below]{} (s0);
  \end{tikzpicture}
  \caption{MDP diagram for the gamble-or-guarantee game.}\label{fig:A-3-risk_seeking}
\end{figure}
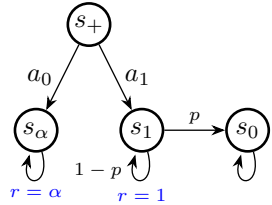

The above game can be formulated as an MDP with state space $\S = \brac{s_{\alpha}, s_0, s_1, s_+}$ and action space $\A = \brac{a_0, a_1}$. Transitions may be probabilistic as indicated by the arrows, and the self-loops at states $s_{\alpha}$, $s_1$, $s_0$ yield rewards $\alpha$, $1$ and $0$, respectively. For the low-rank representation, set the feature dimension to be $d = 5$, and define the feature vectors by
\begin{equation}
  \bphi^{\circ}_h(s_+, a_0) = \begin{bmatrix}
    1 \\ 0 \\ 0 \\ 0 \\ 0
  \end{bmatrix},~
  \bphi^{\circ}_h(s_+, a_1) = \begin{bmatrix}
    0 \\ 1 \\ 0 \\ 0 \\ 0
  \end{bmatrix},~
  \bphi^{\circ}_h(s_{\alpha}, \cdot) = \begin{bmatrix}
    0 \\ 0 \\ 1 \\ 0 \\ 0
  \end{bmatrix},~
  \bphi^{\circ}_h(s_1, \cdot) = \begin{bmatrix}
    0 \\ 0 \\ 0 \\ 1 \\ 0
  \end{bmatrix},~
  \bphi^{\circ}_h(s_0, \cdot) = \begin{bmatrix}
    0 \\ 0 \\ 0 \\ 0 \\ 1
  \end{bmatrix},
\end{equation}
with corresponding factors (for the sake of convenience, unspecified factors are set to $\bm{0}$ by default)
\begin{equation}
  \bnu^{\circ}_h = \begin{bmatrix}
    0 \\ 0 \\ \alpha \\ 1 \\ 0
  \end{bmatrix};\quad
  \bmu^{\circ}_h(s_{\alpha}) = \begin{bmatrix}
    1 \\ 0 \\ 1 \\ 0 \\ 0
  \end{bmatrix},~
  \bmu^{\circ}_h(s_1) = \begin{bmatrix}
    0 \\ 1 \\ 0 \\ 1-p \\ 0
  \end{bmatrix},~
  \bmu^{\circ}_h(s_0) = \begin{bmatrix}
    0 \\ 0 \\ 0 \\ p \\ 1
  \end{bmatrix}.
\end{equation}
We specify the $(\xi,\eta)$-ambiguity set $\widehat{\M}$ (for the low-rank robust evaluation) by radii $\r_{\xi,h} = \delta$ and $\r_{\eta,h} = 0$.

Now we shall calculate the optimal policies in the nominal and low-rank robust MDPs. Here we assume that the horizon $H$ is sufficiently large, such that $(1-p)^{H-1} < \frac{1}{2}$, $(1-\delta)^{H-1} < \frac{1}{2}$ and $\alpha < \frac{1-p}{2pH}$.
\begin{itemize}
  \item \textit{Optimal nominal policy.} Note that $\Vnom^*_h(s_0) = 0$ for all $h \in [H]$, by Bellman optimality equation we have $\Vnom^*_h(s_1) = (1-p) \prn[\big]{V_{h+1}(s_1) + 1}$, which gives $\Vnom^*_h(s_1) = \frac{1-p}{p} \prn[\big]{1 - (1-p)^{H-h+1}}$. Meanwhile, we also have $\Vnom^*_h(s_{\alpha}) = (H-h+1) \alpha$. By assumption, $\Vnom^*_2(s_1) > \Vnom^*_2(s_{\alpha})$, and thus the optimal nominal policy $\pi^*$ is to take action $a_1$ at the initial state $s_+$.

  \item \textit{Optimal low-rank robust policy.} Since the action at $s_+$ is the only decision made by the agent, in this special case we shall still study $\Vrlx^*_2(s_1)$ and $\Vrlx^*_2(s_{\alpha})$. Note that the worst-case perturbation at $s_{\alpha}$ is equivalent to increasing the transition probabilities from $s_{\alpha}$ to $s_0$ by $\delta$, so that we have $\Vrlx^*_h(s_{\alpha}) = \frac{(1-\delta)\alpha}{\delta} \prn[\big]{1 - (1-\delta)^{H-h+1}}$. On the other hand, an admissible perturbation at $s_1$ is to increase the transition probability from $s_1$ to $s_0$ by $\delta$, so we have $\Vrlx^*_h(s_1) \leq \frac{1-p-\delta}{p+\delta}$. Therefore, when $\alpha > \frac{2 \delta (1-p-\delta)}{(1-\delta)(p+\delta)}$ (or as a sufficient condition, when $\delta < \frac{p\alpha}{2(1-p)}$), we have $\Vrlx^*_2(s_1) < \Vrlx^*_2(s_{\alpha})$, and thus the optimal standard robust policy is to take action $a_0$ at intial state $s_+$.
\end{itemize}
The above example indicates that, with appropriate perturbation radius, the policy obtained by solving \eqref{eq:3-effective_robust_policy} indeed displays certain level of robust behavior.
  \section{Reduction of QC2QPs to Constrained SDPs}

In this section we include the details of the reduction from \textit{Quadratic Program with \textit{Two} Quadratic Constraints (QC2QP)} to \textit{Semi-Definite Program (SDP)}. Note that QC2QP, and more generally all quadratic programs with quadratic constraints (QCQP), are extensively studied in literature, since they are closely related to the trust-region method \citep{beck2006strong, ai2009strong, chen2021optimality}.

We have the following sufficient condition for \textit{strong duality} to hold between a QC2QP and its corresponding relaxed SDP.

\begin{theorem}[\citet{beck2006strong}, Theorem 3.5]\label{thm:B-1-QC2QP_duality}
  Given a QC2QP in the form
  \begin{subequations}\label{eq:B-1-QC2QP}
  \begin{align}
    \min_{\bm{x} \in \R^n}\quad& \bm{x}^{\top} \bm{A} \bm{x} + 2 \bm{\beta}^{\top} \bm{x} + c,\\
    \mathrm{s.t.}\quad& \bm{x}^{\top} \bm{A}_1 \bm{x} + 2 \bm{\beta}_1^{\top} \bm{x} + c_1 \geq 0,\\
                      & \bm{x}^{\top} \bm{A}_2 \bm{x} + 2 \bm{\beta}_2^{\top} \bm{x} + c_2 \geq 0,
  \end{align}
  \end{subequations}
  let its SDP relaxation be
  \begin{subequations}\label{eq:B-1-SDP}
  \begin{align}
    \min_{\bm{X} \in S^{n+1}}&\quad \tr(\bm{C} \bm{X}) \\
    \mathrm{s.t.}&\quad \tr(\bm{C}_1 \bm{X}) \leq 0,\quad
                        \tr(\bm{C}_2 \bm{X}) \leq 0,\\
                 &\quad \bm{X}_{n+1, n+1} = 1,\quad
                        \bm{X} \succeq 0,
  \end{align}
  \end{subequations}
  where $S^n$ denotes the set of $n$-by-$n$ real-symmetric matrices. Suppose the following conditions hold:
  \begin{itemize}
    \item the QC2QP \eqref{eq:B-1-QC2QP} is strictly feasible, \textit{i.e.}, there exists an $\bm{x}_0$ such that the inequality constraints are strict;
    \item there exists $\alpha_1, \alpha_2 \in \R$, such that $\alpha_1 \bm{A}_1 + \alpha_2 \bm{A}_2 \succ 0$;
    \item the dimension of the null space $\dim \mathcal{N}(\bm{A} - \alpha^* \bm{A}_1 - \beta^* \bm{A}_2) \neq 1$, where $(\lambda^*, \alpha^*, \beta^*)$ is the solution to the dual program,
  \end{itemize}
  then strong duality holds for the QC2QP \eqref{eq:B-1-QC2QP}. Consequently, the optimal values of \eqref{eq:B-1-QC2QP} and \eqref{eq:B-1-SDP} are identical.
\end{theorem}

Now we are ready to show \Cref{thm:4-2-reduction} using \Cref{thm:B-1-QC2QP_duality}.

\begin{proof}[Proof of \Cref{thm:4-2-reduction}]
  Note that \eqref{eq:4-2-optim_QC2QP} can be further equivalently rewritten with an auxiliary matrix $\bm{Z} = \bm{z}\bm{z}^{\top}$ as
  \begin{subequations}\label{eq:4-2-QC2QP_primal}
  \begin{align}
    \min_{\bm{z}, \bm{Z}}&\quad \tr(\bm{A} \bm{Z}) + 2 \bm{\beta}^{\top} \bm{z} + c \label{eq:4-2-QC2QP_primal:1}\\
    \mathrm{s.t.}&\quad \tr(\bm{A}_x \bm{Z}) - R_x^2 \leq 0, \label{eq:4-2-QC2QP_primal:2}\\
                 &\quad \tr(\bm{A}_y \bm{Z}) - R_y^2 \leq 0, \label{eq:4-2-QC2QP_primal:3}\\
                 &\quad \bm{Z} = \bm{z} \bm{z}^{\top}. \label{eq:4-2-QC2QP_primal:4}
  \end{align}
  \end{subequations}
  Therefore, if we partition $\bm{X}$ into blocks as
  \( \brak*{\begin{smallmatrix}
      \bm{Z} & \bm{z} \\
      \bm{z}^{\top} & 1
    \end{smallmatrix}} \),
  the objective function in \eqref{eq:4-2-QC2QP_primal:1} and the constraints \eqref{eq:4-2-QC2QP_primal:2}, \eqref{eq:4-2-QC2QP_primal:3} would be equivalent to \eqref{eq:4-2-QC2QP_dual-SDP:1} and \eqref{eq:4-2-QC2QP_dual-SDP:2}, respectively. Meanwhile, by properties of Schur complement, $\bm{X} \succeq 0$ is equivalent to $\bm{Z} \succeq \bm{z} \bm{z}^{\top}$, which is a relaxation of \eqref{eq:4-2-QC2QP_primal:4}. Therefore, \eqref{eq:4-2-QC2QP_dual-SDP} is an SDP relaxation of \eqref{eq:4-2-optim_QC2QP}.

  It only suffices to verify that strong duality holds for this SDP \eqref{eq:4-2-QC2QP_dual-SDP}, so that the optimal $\bm{z}^*$ for \eqref{eq:4-2-optim_QC2QP} can be recovered from optimal $\bm{X}^*$ for \eqref{eq:4-2-QC2QP_dual-SDP}. For this purpose, we only need to verify the sufficient conditions required by \Cref{thm:B-1-QC2QP_duality}:
  \begin{itemize}
    \item The inequality constraints are strictly feasible at, e.g., $\bm{z}_0 = \bm{0}$.
    \item $\bm{A}_x + \bm{A}_y = \bm{I}_{2d} \succ 0$.
    \item For any $\alpha, \beta \in \R$, since the determinant of $\bm{A} - \alpha \bm{A}_x - \beta \bm{A}_y$ is $(\alpha \beta - \frac{1}{4})^d$, the dimension of the null space is either $0$ (when $\alpha \beta \neq \frac{1}{4}$) or $d$ (when $\alpha \beta = \frac{1}{4}$).
  \end{itemize}
  Therefore, the theorem directly follows from \Cref{thm:B-1-QC2QP_duality}.
\end{proof}

Although the exact solution of SDPs is still computationally hard, there are a series of approximation algorithms that return good approximate solutions within reasonable time \citep{hazan2008sparse, gartner2012approximation}.

  \section{Theoretical Analysis}\label{sec:appx-C-proof}

In this section, we present the complete proof for the convergence analysis in \Cref{sec:5-analysis}. For the convenience of readers, we also repeat the statement of the results.

Throughout this section we assume the $(\xi,\eta)$-ambiguity set $\hat{\M}$ is rectangular with radii $(\r_{\xi,h}, \r_{\eta,h})$.

\subsection{Quasi-contraction Property}

We first show an important lemma that characterizes the \textit{quasi}-contraction property of the effective robust Bellman operator $\Vrlx^{\pi}_h$, which is similar to the contraction property of standard robust Bellman operators (see Theorem 3.2(a) in \citet{iyengar2005robust}). The difference lies in the fact that \eqref{eq:C-1-contraction:e0} only holds when we take expectation over $(s,a) \sim d^{\pi}_h$, and there is an additional constant term due to the gap between a pseudo-MDP obtained by perturbation and the closest MDP from it.

\begingroup
\def\thelemma{\ref{thm:5-contraction}}
\begin{lemma}
  For any step $h \in [H]$, any $V$-functions $V, V': \S \to \R$ and any policy $\pi$, we have
  \begin{equation}\label{eq:C-1-contraction:e0}
    \E[(s,a) \sim d^{\pi}_h]{[\Brlx^{\pi}_h V](s,a) - [\Brlx^{\pi}_h V'](s,a)} \leq \E[s' \sim \rho^{\pi}_{h+1}]{V(s') - V'(s')} + 2 \r_{\eta,h} \sqrt{d}.
  \end{equation}
\end{lemma}
\addtocounter{theorem}{-1}
\endgroup

\begin{proof}
  Let $(\bm{\eta}^*_{V}, \bxi^*_{V})$, $(\bm{\eta}^*_{V'}, \bxi^*_{V'})$ be the optimal solutions to \eqref{eq:3-effective_robust_update} that attain minimum with respect to $V$, $V'$, respectively. For the sake of convenience, denote by $\bomega^{\circ}_{V} := \bnu^{\circ}_h + \sum_{s'} V(s') \bmu^{\circ}_h(s')$ and $\bomega^{\circ}_{V'} := \bnu^{\circ}_h + \sum_{s'} V'(s') \bmu^{\circ}_h(s')$ the parameters for the $Q$-functions obtained by updating $V$, $V'$ in the nominal feature space, respectively. Then we have
  \begin{subequations}\label{eq:C-1-contraction:e1}
  \begin{align}
    &\E[(s,a) \sim d^{\pi}_h]{[\Brlx^{\pi}_h V](s,a)} - \E[(s,a) \sim d^{\pi}_h]{[\Brlx^{\pi}_h V'](s,a)} \nonumber\\
    ={}& \angl*{\E[(s,a) \sim d^{\pi}_h]{\bphi^{\circ}(s,a)} + \bm{\eta}^*_{V}, \bomega^{\circ}_{V} + \bxi^*_{V}} - \angl*{\E[(s,a) \sim d^{\pi}_h]{\bphi^{\circ}(s,a)} + \bm{\eta}^*_{V'}, \bomega^{\circ}_{V'} + \bxi^*_{V'}} \label{eq:C-1-contraction:e1:1}\\
    ={}& \min_{\bm{\eta}_h, \bxi_h} \angl*{\E[(s,a) \sim d^{\pi}_h]{\bphi^{\circ}(s,a)} + \bm{\eta}_h, \bomega^{\circ}_{V} + \bxi_h} - \angl*{\E[(s,a) \sim d^{\pi}_h]{\bphi^{\circ}(s,a)} + \bm{\eta}^*_{V'}, \bomega^{\circ}_{V'} + \bxi^*_{V'}} \label{eq:C-1-contraction:e1:2}\\
    \leq{}& \angl*{\E[(s,a) \sim d^{\pi}_h]{\bphi^{\circ}(s,a)} + \bm{\eta}^*_{V'}, \bomega^{\circ}_{V} + \bxi^*_{V'}} - \angl*{\E[(s,a) \sim d^{\pi}_h]{\bphi^{\circ}(s,a)} + \bm{\eta}^*_{V'}, \bomega^{\circ}_{V'} + \bxi^*_{V'}} \label{eq:C-1-contraction:e1:3}\\
    ={}& \E[(s,a) \sim d^{\pi}_h]{\angl*{\bphi^{\circ}(s,a), \bomega^{\circ}_{V} - \bomega^{\circ}_{V'}}} + \angl*{\bm{\eta}^*_{V'}, \bomega^{\circ}_{V} - \bomega^{\circ}_{V'}} \label{eq:C-1-contraction:e1:4}\\
    ={}& \EE[(s,a) \sim d^{\pi}_h][s' \sim \P^{\circ}(\cdot | s,a)]{V(s') - V'(s')} + \angl*{\bm{\eta}^*_{V'}, \bomega^{\circ}_{V} - \bomega^{\circ}_{V'}} \label{eq:C-1-contraction:e1:5}\\
    \leq{}& \E[s' \sim \rho^{\pi}_{h+1}]{V(s') - V'(s')} + 2\r_{\eta,h} \sqrt{d}. \label{eq:C-1-contraction:e1:6}
  \end{align}
  \end{subequations}
  Here in \eqref{eq:C-1-contraction:e1:1} we plug in the definition of $\Brlx^{\pi}_h$ and apply linearity of inner product; in \eqref{eq:C-1-contraction:e1:2} and \eqref{eq:C-1-contraction:e1:3} we use the optimality of $(\bm{\eta}^*_{V}, \bxi^*_{V})$ with respect to $V$ in \eqref{eq:3-effective_robust_update}; in \eqref{eq:C-1-contraction:e1:4} we cancel out and rearrange the terms; in \eqref{eq:C-1-contraction:e1:5} we apply the fact that $\bomega^{\circ}_{V} - \bomega^{\circ}_{V'} = \sum_{s'} \prn[\big]{V(s') - V'(s')} \bmu^{\circ}_h(s')$, and further, $\angl*{\bphi^{\circ}(s,a), \bomega^{\circ}_{V} - \bomega^{\circ}_{V'}} = \sum_{s'} \prn[\big]{V(s') - V'(s')} \angl*{\bphi^{\circ}(s,a), \bmu^{\circ}(s')} = \E[s' \sim \P^{\circ}(\cdot | s,a)]{V(s') - V'(s')}$; in \eqref{eq:C-1-contraction:e1:5} we plug in the relation between state and state-action occupancy measures, and use the fact that $\norm{\bm{\eta}^*_{V'}} \leq \r_{\eta,h}$ and $\norm{\bomega^{\circ}_{V} - \bomega^{\circ}_{V'}} = \norm{\sum_{s'} \prn[\big]{V(s') - V'(s')} \bmu^{\circ}(s')} \leq 2\sqrt{d}$ (by \Cref{assm:2-low-rank_MDP_bound}).
\end{proof}

In contrast to the contraction properties of nominal and standard robust Bellman operators, here we have an additional constant term that is proportional to the perturbation radii of $\bm{\eta}_{1:H}$. This is actually intrinsic for pseudo-MDPs due to the misspecification error. Indeed, in \cite{yao2014pseudo} the bound also carries a constant that characterizes the misspecification error against MDPs.

It is also worth mentioning that \Cref{thm:5-contraction} helps to justify why we are allowed to ``roll out'' the trajectory to the $h$\tsup{th} step with policy $\pi$ in the \textit{nominal} system --- the influence of the perturbation is completely absorbed into the constant term.

\subsection{Performance Difference Lemma}

We proceed to show an extended version of the famous Performance Difference Lemma that characterizes the difference between policies in terms of low-rank robust values. The lemma is similar to a few other extended performance difference lemmas in literature (see, e.g., Lemma 1 in \citet{efroni2020optimistic}).

\begingroup
\def\thelemma{\ref{thm:5-performance_difference}}
\begin{lemma}[Extended Performance Difference Lemma]
  For any policies $\pi$ and $\pi'$, we have
  \begin{align}\label{eq:C-1-performance_difference:e0}
    &\Vrlx_1^{\pi}(\rho) - \Vrlx_1^{\pi'}(\rho)
    \leq{} \\
    &\hspace{3em} \sum_{h=1}^{H} \E[\pi, \P^{\circ}]{\angl*{\Qrlx_h^{\pi'}(s_h,\cdot), \pi_h(\cdot | s_h) - \pi'_h(\cdot | s_h)} + \angl*{[\Brlx^{\pi}_h \Vrlx_{h+1}^{\pi'}](s_h,\cdot) - \Qrlx_h^{\pi'}(s_h,\cdot), \pi_h(\cdot | s_h)}} + 2 \sum_{h=1}^{H}\r_{\eta,h} \sqrt{d}. \nonumber
  \end{align}
\end{lemma}
\addtocounter{theorem}{-1}
\endgroup

\begin{proof}
  By definition of the effective robust value functions, we have
  \begin{subequations}\label{eq:C-1-performance_difference:e1}
  \begin{align}
    \E[s \sim \rho^{\pi}_h]{\Vrlx_h^{\pi}(s) - \Vrlx_h^{\pi'}(s)}
    &= \E[s \sim \rho^{\pi}_h]{\angl[\Big]{\Qrlx_h^{\pi}(s,\cdot), \pi_h(\cdot | s)} - \angl[\Big]{\Qrlx_h^{\pi'}(s,\cdot), \pi'_h(\cdot | s)}} \label{eq:C-1-performance_difference:e1:1}\\
    &= \E[s \sim \rho^{\pi}_h]{\angl*{\Qrlx_h^{\pi'}(s,\cdot), \pi_h(\cdot | s) - \pi'_h(\cdot | s)} + \angl*{\Qrlx_h^{\pi}(s,\cdot) - \Qrlx_h^{\pi'}(s,\cdot), \pi_h(\cdot | s)}} \label{eq:C-1-performance_difference:e1:2}\\
    &= \E[s \sim \rho^{\pi}_h]{\angl*{\Qrlx_h^{\pi'}(s,\cdot), \pi_h(\cdot | s) - \pi'_h(\cdot | s)} + \angl*{[\Brlx^{\pi}_h \Vrlx_{h+1}^{\pi}](s,\cdot) - \Qrlx^{\pi'}_h(s,\cdot), \pi_h(\cdot | s)}} \label{eq:C-1-performance_difference:e1:3}\\
    &= \E[s \sim \rho^{\pi}_h]{\angl*{\Qrlx_h^{\pi'}(s,\cdot), \pi_h(\cdot | s) - \pi'_h(\cdot | s)} + \angl*{[\Brlx^{\pi}_h \Vrlx_{h+1}^{\pi'}](s,\cdot) - \Qrlx_h^{\pi'}(s,\cdot), \pi_h(\cdot | s)}} \nonumber\\
    &\hspace{4em} {}+ \E[s \sim \rho^{\pi}_h]{\angl*{[\Brlx^{\pi}_h \Vrlx_{h+1}^{\pi}](s,\cdot) - [\Brlx^{\pi}_h \Vrlx_{h+1}^{\pi'}](s,\cdot), \pi_h(\cdot | s)}} \label{eq:C-1-performance_difference:e1:4}\\
    &= \E[s \sim \rho^{\pi}_h]{\angl*{\Qrlx_h^{\pi'}(s,\cdot), \pi_h(\cdot | s) - \pi'_h(\cdot | s)} + \angl*{[\Brlx^{\pi}_h \Vrlx_{h+1}^{\pi'}](s,\cdot) - \Qrlx_h^{\pi'}(s,\cdot), \pi_h(\cdot | s)}} \nonumber\\
    &\hspace{4em} {}+ \E[(s,a) \sim d^{\pi}_h]{[\Brlx^{\pi}_h \Vrlx_{h+1}^{\pi}](s,a) - [\Brlx^{\pi}_h \Vrlx_{h+1}^{\pi'}](s,a)} \label{eq:C-1-performance_difference:e1:5}\\
    &\leq \E[s \sim \rho^{\pi}_h]{\angl*{\Qrlx_h^{\pi'}(s,\cdot), \pi_h(\cdot | s) - \pi'_h(\cdot | s)} + \angl*{[\Brlx^{\pi}_h \Vrlx_{h+1}^{\pi'}](s,\cdot) - \Qrlx_h^{\pi'}(s,\cdot), \pi_h(\cdot | s)}} \nonumber\\
    &\hspace{4em} {}+ \E[s' \sim \rho^{\pi}_{h+1}]{\Vrlx_{h+1}^{\pi}(s') - \Vrlx_{h+1}^{\pi'}(s')} + 2\r_{\eta,h} \sqrt{d}, \label{eq:C-1-performance_difference:e1:6}
  \end{align}
  \end{subequations}
  where in \eqref{eq:C-1-performance_difference:e1:1} we plug in the relationship between $V$- and $Q$-functions; in \eqref{eq:C-1-performance_difference:e1:2} through \eqref{eq:C-1-performance_difference:e1:4} we rearrange the terms and plug in the effective robust Bellman operator; in \eqref{eq:C-1-performance_difference:e1:5} we use the fact that $d^{\pi}_h(s,a) = \rho^{\pi}_h(s) \pi(a | s)$; and in \eqref{eq:C-1-performance_difference:e1:6} we apply \Cref{thm:5-contraction}.
  
  Note that in \eqref{eq:C-1-performance_difference:e1:6} the term on the left-hand side recursively appears with subscript $h+1$. Therefore, by induction over $h$ and the conventional boundary condition $\Vrlx^{\pi}_{H+1}(\cdot) = \Vrlx^{\pi'}_{H+1}(\cdot) = 0$, we shall derive \eqref{eq:C-1-performance_difference:e0}. Note that we also rewrite the expectation using the fact that sampling $s_h \sim \rho^{\pi}_h$ is equivalent to rolling out the trajectory using policy $\pi$ in the nominal model. This completes the proof.
\end{proof}

We would like to point out again that the proposed extended version of Performance Difference Lemma is favorable in that the trajectory is rolled out with respect to the \textit{nominal} model, which greatly simplifies the subsequent analysis.

\subsection{Natural Policy Gradient}\label{sec:appx-NPG}

In each iteration of the proposed algorithm, after the robust evaluation of a given policy $\pi^k$, a subsequent policy update step improves the policy by the \textit{Natural Policy Gradient (NPG)} algorithm. The NPG algorithm in the episodic setting is closely related to the \textit{expert-advice} problem that is well-studied in online learning --- the update rule of the policy $\pi^k_h(\cdot | s)$ shall be interpreted as an expert-advice instance at each $(h,s) \in [H] \times \S$, where the loss of action $a$ (the ``expert'') is set to be $g^k_a := -\Qrlx^{\pi^k}_h(s,a)$, and the cost of decision $\bm{x} = \pi^k_h(\cdot | s)$ is set to be $\ell(\bm{x}^k) := \angl{\bm{g}^k, \bm{x}^k}$ (here we regard $\bm{g}^k \in \Delta(\A)$ as a vector in $\R^A$). Then we can use a specific variant of \textit{Online Mirror Descent (OMD)}, i.e. \textit{Exponentiated Gradient Descent}, to solve the problem. The regret of the algorithm can be characterized as the following theorem, which is well-known in literature (see, e.g., Section 6.6 in \cite{orabona2019modern}).

\begin{theorem}[Exponentiated Gradient Descent for expert-advice problem]
  For an expert-advice problem with expert set $\A$, where the expert cost vector is $\bm{g}^k \in \R^A$ in round $k$, suppose the decision $\bm{x}^k$ is updated by
  \begin{subequations}
  \begin{align}
    x^1_a &\gets \tfrac{1}{A},~ \forall a \in \A;\\
    x^{k+1}_a &\gets x^k_a \cdot \exp(-\alpha g^k_a),~ \forall a \in \A.
  \end{align}
  \end{subequations}
  Then the cumulative regret against any static decision $\bm{x} \in \R^A$ is upper bounded by
  \begin{equation}
    \sum_{k=1}^{K} \angl{\bm{g}^k, \bm{x}^k - \bm{x}} \leq \frac{\log d}{\alpha} + \frac{\alpha}{2} \sum_{k=1}^{K} \norm{\bm{g}^k}_{\infty}^2.
  \end{equation}
\end{theorem}

Now we shall switch to the standard RL notations to obtain the following corollary.

\begin{corollary}\label{thm:C-2-expert_regret_RL}
  For any $(h,s) \in [H] \times \S$, suppose the update of $\pi^k_h(\cdot | s)$ follows the rule specified in \Cref{alg:4-main}, then
  \begin{equation}
    \sum_{k=1}^{K} \angl*{\Qrlx_h^{\pi^k}(s,\cdot), \pirlx^*(\cdot | s) - \pi^k_h(\cdot | s)}
    \leq \frac{\log A}{\alpha} + \frac{\alpha}{2} \sum_{k=1}^{K} \norm*{Q^k_h(s,\cdot)}_{\infty}^2
  \end{equation}
\end{corollary}

\subsection{Convergence}

Finally, we conclude the proof by following a similar proof strategy as in \citep{liu2023optimistic} to show the convergence of \Cref{alg:4-main}. For this purpose, we also need the following technical lemma that characterizes the difference between robust Bellman operators with respect to different policies.

\begin{lemma}\label{thm:C-3-operator_difference}
  For any step $h$, any $V$-function $V: \S \to \R$ and any policies $\pi, \pi'$, we have
  \begin{equation}\label{eq:C-2-operator_difference:e0}
    [\Brlx^{\pi}_h V](s,a) - [\Brlx^{\pi'}_h V](s,a) \leq 2\r_{\xi,h} (1+\r_{\eta,h}) + 4\r_{\eta,h} \sqrt{d},~ \forall (s,a) \in \S \times \A.
  \end{equation}
\end{lemma}

\begin{proof}
  Let $(\bm{\eta}^*_{\pi}, \bxi^*_{\pi})$, $(\bm{\eta}^*_{\pi'}, \bxi^*_{\pi'})$ be the optimal solutions to \eqref{eq:3-effective_robust_update} that attain minimum with respect to effective robust Bellman operators $\Brlx^{\pi}_h$, $\Brlx^{\pi'}_h$, respectively. Then for any $(s,a) \in \S \times \A$, we have
  \begin{subequations}\label{eq:C-2-operator_difference:e1}
  \begin{align}
    [\Brlx^{\pi}_h V](s,a) - [\Brlx^{\pi'}_h V](s,a)
    &= \angl*{\bphi^{\circ}_h(s,a) + \bm{\eta}^*_{\pi}, \bomega^{\circ}_h + \bxi^*_{\pi}} - \angl*{\bphi^{\circ}_h(s,a) + \bm{\eta}^*_{\pi'}, \bomega^{\circ}_h + \bxi^*_{\pi'}} \label{eq:C-2-operator_difference:e1:1}\\
    &= \angl*{\bm{\eta}^*_{\pi} - \bm{\eta}^*_{\pi'}, \bm{\bomega}^{\circ}_h} + \angl*{\bphi^{\circ}_h(s,a), \bxi^*_{\pi} - \bxi^*_{\pi'}} + \angl*{\bm{\eta}^*_{\pi}, \bxi^*_{\pi}} - \angl*{\bm{\eta}^*_{\pi'}, \bxi^*_{\pi'}} \label{eq:C-2-operator_difference:e1:2}\\
    &\leq \norm{\bm{\eta}^*_{\pi} - \bm{\eta}^*_{\pi'}} \norm{\bm{\bomega}^{\circ}_h} + \norm{\bphi^{\circ}_h(s,a)} \norm{\bxi^*_{\pi} - \bxi^*_{\pi'}} + \norm{\bm{\eta}^*_{\pi}} \norm{\bxi^*_{\pi}} + \norm{\bm{\eta}^*_{\pi'}} \norm{\bxi^*_{\pi'}} \label{eq:C-2-operator_difference:e1:3}\\
    &\leq 2\r_{\xi,h} (1+\r_{\eta,h}) + 4\r_{\eta,h} \sqrt{d}, \label{eq:C-2-operator_difference:e1:4}
  \end{align}
  \end{subequations}
  where in \eqref{eq:C-2-operator_difference:e1:1} we plug in the definition of effective robust Bellman operators; in \eqref{eq:C-2-operator_difference:e1:2} and \eqref{eq:C-2-operator_difference:e1:3} we rearrange the terms and apply triangle inequality; in \eqref{eq:C-2-operator_difference:e1:4} we plug in $(\xi,\eta)$-rectangularity, \Cref{assm:2-low-rank_MDP_bound} and its corollary that $\norm{\bomega^{\circ}} = \norm{\bnu^{\circ}_h + \sum_{s'} V(s') \bmu^{\circ}_h(s')} \leq 2\sqrt{d}$. This completes the proof.
\end{proof}

\begingroup
\def\thetheorem{\ref{thm:5-convergence}}
\begin{theorem}
  Under \Cref{assm:2-low-rank_MDP_bound}, by running \Cref{alg:4-main} with parameter $\alpha = \sqrt{2\log A / (KH^2)}$, the effective robust $V$-function of the output policy $\pi^{\mathrm{out}}$ satisfies
  \begin{equation}\label{eq:C-3-convergence:e0}
    \E[\pi^{\mathrm{out}}]{\Vrlx_1^*(\rho) - \Vrlx_1^{\pi^{\mathrm{out}}}(\rho)}
    \leq \sqrt{\frac{2H^4 \log A}{K}} + \sum_{h=1}^{H} \prn[\big]{2\r_{\xi,h} (1+\r_{\eta,h}) + 6\r_{\eta,h} \sqrt{d}}.
  \end{equation}
\end{theorem}
\addtocounter{theorem}{-1}
\endgroup

\begin{proof}
  Let $\pirlx^*$ be the effective robust optimal policy defined in \eqref{eq:3-effective_robust_policy}, and write $\Vrlx^*_h(s)$, $\Qrlx^*_h(s)$ for its effective robust value functions as a shorthand. Then, according to the algorithm, we have
  \begin{subequations}\label{eq:C-3-convergence:e1}
  \begin{align}
    &\E[\pi^{\mathrm{out}}]{\Vrlx_1^*(\rho) - \Vrlx_1^{\pi^{\mathrm{out}}}(\rho)}
    = \frac{1}{K} \sum_{k=1}^{K} \brak*{\Vrlx_1^*(\rho) - \Vrlx_1^{\pi^k}(\rho)} \label{eq:C-3-convergence:e1:1}\\
    \leq{}& \frac{1}{K} \sum_{k=1}^{K} \sum_{h=1}^{H} \E[\pi^k, \P^{\circ}]{\angl*{\Qrlx_h^{\pi^k}(s_h,\cdot), \pirlx^*(\cdot | s_h) - \pi^k_h(\cdot | s_h)} + \angl*{[\Brlx^{\pirlx^*}_h \Vrlx_{h+1}^{\pi^k}](s_h,\cdot) - \Qrlx_h^{\pi^k}(s_h,\cdot), \pirlx^*_h(\cdot | s_h)}} + 2H\r_{\eta,h} \sqrt{d} \label{eq:C-3-convergence:e1:2}\\
    \leq{}& \frac{H}{K} \max_{h,s}\brac*{\sum_{k=1}^{K} \angl*{\Qrlx_h^{\pi^k}(s,\cdot), \pirlx^*(\cdot | s) - \pi^k_h(\cdot | s)}} + \sum_{h=1}^{H} \prn[\big]{2\r_{\xi,h} (1+\r_{\eta,h}) + 4\r_{\eta,h} \sqrt{d}} + 2 \sum_{h=1}^{H} \r_{\eta,h} \sqrt{d}, \label{eq:C-3-convergence:e1:3}
  \end{align}
  \end{subequations}
  where \eqref{eq:C-3-convergence:e1:1} is dictated by the algorithm; in \eqref{eq:C-3-convergence:e1:2} we apply \Cref{thm:5-performance_difference} with respect to $\pi \gets \pirlx^*$ and $\pi' \gets \pi^k$; while in \eqref{eq:C-3-convergence:e1:3} we take the maximum over the first term and apply \Cref{thm:C-3-operator_difference} for the second term.
  
  It only suffices to bound the first term, for which we shall apply \Cref{thm:C-2-expert_regret_RL} to obtain
  \begin{equation}
    \max_{h,s}\brac*{\sum_{k=1}^{K} \angl*{\Qrlx_h^{\pi^k}(s,\cdot), \pirlx^*(\cdot | s) - \pi^k_h(\cdot | s)}}
    \leq \frac{\log A}{\alpha} + \frac{\alpha}{2} \sum_{k=1}^{K} \norm*{Q^k_h(s,\cdot)}_{\infty}^2
    \leq \frac{\log A}{\alpha} + \frac{\alpha KH^2}{2},
  \end{equation}
  where we also use the fact $Q \in [0, H]$. Now we shall take $\alpha = \sqrt{2\log A / (KH^2)}$ to derive
  \begin{equation}
     \max_{h,s}\brac*{\sum_{k=1}^{K} \angl*{\Qrlx_h^{\pi^k}(s,\cdot), \pirlx^*(\cdot | s) - \pi^k_h(\cdot | s)}} \leq \sqrt{2KH^2 \log A},
  \end{equation}
  and consequently,
  \begin{equation}
    \E[\pi^{\mathrm{out}}]{\Vrlx_1^*(\rho) - \Vrlx_1^{\pi^{\mathrm{out}}}(\rho)}
    \leq \sqrt{\frac{2H^4 \log A}{K}} + \sum_{h=1}^{H} \prn[\big]{2\r_{\xi,h} (1+\r_{\eta,h}) + 6\r_{\eta,h} \sqrt{d}}.
  \end{equation}
  This completes the proof.
\end{proof}
  \section{Numerical Simulations}\label{sec:appx-D-simulation}

\subsection{A Toy Model}

In this section, we study the numerical performance of our R\tsup{2}PG algorithm on a toy model. We show by experimental results that the optimal nominal policy may be sensitive to small perturbations on the MDP, which highlights the need for robust MDPs. We also compare the performance of the output policy against the optimal standard robust policy to reveal the difference between different robustness concepts.

\begin{minipage}[b]{0.36\linewidth}
  \centering
  \begin{figure}[H]
    \centering
      \tikzset{snode/.style = {draw=black, shape=circle, line width=1.0pt, inner sep=0pt, minimum width=16pt}}
  \tikzset{apath/.style = {draw=black, line width=0.5pt, ->, >={Stealth}}}
  \begin{tikzpicture}
    \node[snode, fill=red!30!white] (s1) at (0pt, 20pt) {$s_1$};
    \node[snode, fill=yellow!30!white] (s2) at (20pt, 0pt) {$s_2$};
    \node[snode, fill=green!30!white] (s3) at (0pt, -20pt) {$s_3$};
    \node[snode, fill=yellow!30!white] (s4) at (-20pt, 0pt) {$s_4$};

    \draw[apath] (s1) edge[bend left] (s2);
    \draw[apath] (s2) edge[bend left] (s3);
    \draw[apath] (s3) edge[bend left] (s4);
    \draw[apath] (s4) edge[bend left] (s1);

    \draw[apath] (s2) edge[bend left] (s1);
    \draw[apath] (s3) edge[bend left] (s2);
    \draw[apath] (s4) edge[bend left] (s3);
    \draw[apath] (s1) edge[bend left] (s4);
    
    \draw[apath] (s1) edge[loop above] (s1);
    \draw[apath] (s2) edge[loop right] (s2);
    \draw[apath] (s3) edge[loop below] (s3);
    \draw[apath] (s4) edge[loop left] (s4);
  \end{tikzpicture}\vspace{6pt}
    \caption{Setup.}\label{fig:D-simulation_MDP}
  \end{figure}
\end{minipage}
\begin{minipage}[b]{0.6\linewidth}
  \centering
  \begin{figure}[H]
    \centering
    \includegraphics[width=0.6\linewidth]{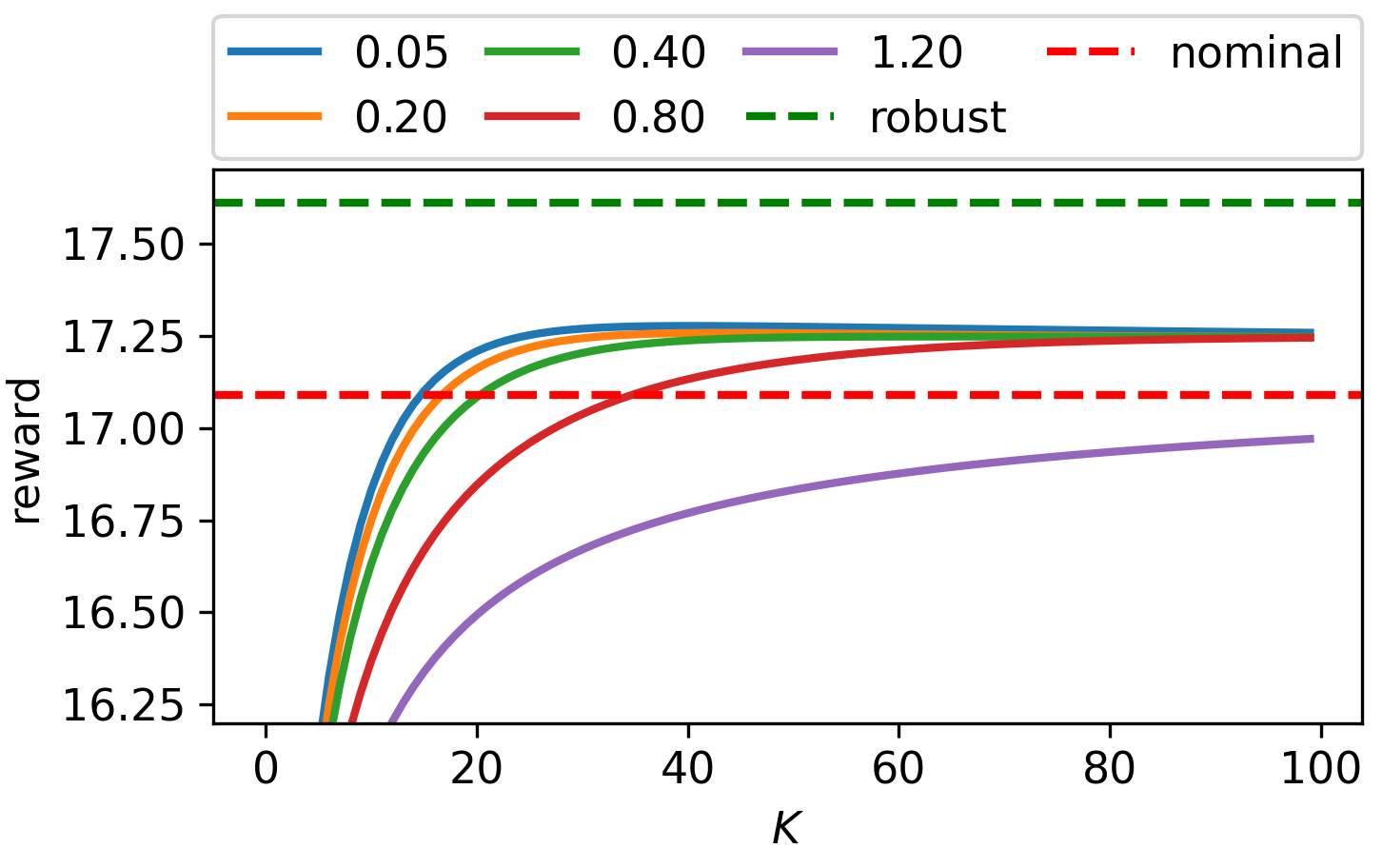}\vspace{-10pt}
    \caption{Output policies with different $\r_{\xi,h}$.}\label{fig:D-simulation_results}
  \end{figure}  
\end{minipage}

\paragraph{Setup.} Consider a tabular MDP with state space $\S = \set{s_1, s_2, s_3, s_4}$ and action space $\A = \set{\curvearrowleft, \downarrow, \curvearrowright}$. In the nominal model, the rewards are set to be
\begin{equation}
  r^{\circ}_h(s_1, \cdot) = 0,\quad
  r^{\circ}_h(s_2, \cdot) = 0.90,\quad
  r^{\circ}_h(s_3, \cdot) = 0.89,\quad
  r^{\circ}_h(s_4, \cdot) = 0.91,\quad
  \forall h \in [H].
\end{equation}
Intuitively, $s_1$ is the 0-reward state to be avoided (marked in red), $s_2$ and $s_4$ are higher-reward states subject to risk (after potential perturbation, marked in yellow), and $s_3$ is the lower-reward safe state (marked in green). The transitions are deterministic, as the arrows suggest --- $\curvearrowleft$ moves 1 step to the counter-clockwise direction, $\downarrow$ stays at the current state, and $\curvearrowright$ moves 1 step to the clockwise direction; formally, the transition probabilities are set to be
\begin{equation}
  \P^{\circ}_h(s_{i-1} | s_i, \curvearrowleft) = \P^{\circ}_h(s_i | s_i, \downarrow) = \P^{\circ}_h(s_{i+1} | s_i, \curvearrowright) = 1,
\end{equation}
where subscripts are understood as modulo 4, i.e. $s_0 = s_4$ and $s_5 = s_1$, while all the unspecified probabilities are set to 0 by default. For the low-rank representation, we use the standard orthonormal feature for tabular MDPs, i.e. $\bphi^{\circ}_h(s,a) = \bm{e}_{(s,a)}$, where $\bm{e}_{(\cdot,\cdot)}$ forms an orthonormal basis in $\R^{S \times A}$. The MDP is illustrated in \Cref{fig:D-simulation_MDP} above.

We may use different methods to solve the MDP, including solving for the optimal nominal policy via dynamic programming, solving for the optimal standard robust policy via robust dynamic programming, and running our R\tsup{2}PG algorithm with different perturbation radii $(\r_{\xi,h}, \r_{\eta,h})$. The output policies are evaluated with respect to a few randomly perturbed MDPs around the nominal model, where all transition probabilities are subject to uncertainty at most $\delta$, and the lowest cumulative reward is reported as the ``empirical robust value'' of the policy.

\paragraph{Results.} The R\tsup{2}PG algorithm is run with different perturbation radii, and the policies obtained in all the episodes are evaluated by the \textit{minimum} cumulative reward evaluated in a few perturbed MDPs. Simulation results for $R_{\eta,h} = 0.01$ and $R_{\xi,h} \in \set{0.05,0.2,0.4,0.8,1.2}$ are plotted in \Cref{fig:6-simulation_results}. It can be observed that policies tend to converge in all executions, and a larger perturbation radius generally leads to more conservative behavior. This phenomenon is largely expected in that, as perturbation radius increases, the misspecification error induced by the worst-case pseudo-MDPs also increases, which leads to an intrinsically pessimistic estimation of policy values. However, the output policies still perform better than the nominal optimal policy when the MDP is appropriately perturbed, highlighting again the need for robustness in environments with uncertainty.

Analyzing the output policies in details, we shall further find that, as time elapses, all output policies gradually lean towards the safe state by increasing the transition probabilities to it. However, since the D\tsup{2}PG algorithm is designed to optimize over the average performance for policy evaluation, it is also reasonable that it does not fully converge to a policy that yields optimal worst-case performance.

\paragraph{Reproducibility.} The code for reproducing the simulation can be found online at \url{https://anonymous.4open.science/r/robust-linear-MDP_ICML-24-249E/}.

\subsection{Inverted Pendulum}

In this section, we show the numerical performance of our robust policy evaluation scheme for a standard nonlinear control task with continuous state-action spaces, which helps to justify the efficiency and realizability of our approach.

\paragraph{Settings.} The robust policy evaluation scheme is tested for Inverted Pendulum, a classic nonlinear control task, the $\delta$-discretized dynamics of which is given by
\begin{equation}
  \underbrace{\begin{bmatrix}
    \theta_{t+1} \\ \dot{\theta}_{t+1}
  \end{bmatrix}}_{s_{t+1}} = \underbrace{\begin{bmatrix}
    \theta_t \\ \dot{\theta}_t
  \end{bmatrix} + \begin{bmatrix}
    \dot{\theta}_t \\
    \frac{3g \sin \theta_t}{2\ell} + \frac{3T_t}{m \ell^2}
  \end{bmatrix} \delta}_{f(s_t, a_t)} {}+ \epsilon_t,
\end{equation}
where the state $s_t := [\theta_t, \dot{\theta}_t]^{\top}$ consists of the angular position $\theta \in [-\pi,\pi]$ and angular velocity $\dot{\theta}$ of the pendulum, and the action $a_t := T_t$ is the torque applied on the pendulum; $\delta$ is the discretization interval, and $\epsilon_t \sim \mathcal{N}(0, \sigma^2 I_n)$ is an i.i.d. Gaussian noise; the mass $m$ and the length $\ell$ of the pendulum are system parameters that are subject to potential perturbations. The reward function is defined as $r(s, a) = -(\theta^2 + 0.01 \dot{\theta}^2 + 0.001 a^2)$, aiming at stabilizing the pendulum at the upright position ($\theta = \dot{\theta} = 0$). The pendulum is initially released at a random position, and the task is to swing it up.

We are concerned about the planning problem in this MDP. For the ease of practical implementation and comparison against the results in \citet{ren2023stochastic}, we consider a slightly different objective of infinite-horizon $\gamma$-discounted cumulative reward. To test the robustness of the output behavior, we perturb the mass $m$ around the nominal value $m^{\circ} = 1.0$ during evaluation, and compare the robust policy against the standard non-robust policy in terms of discounted cumulative reward.



\paragraph{Algorithm.} According to \citet{ren2023stochastic}, given known dynamics $f$ and assuming Gaussian noise, we are able to construct time-invariant feature and factor vectors $\bphi^{\circ}(s,a)$ and $\bmu^{\circ}(s')$ in their closed forms, so that we can skip the learning of representation and focus on displaying the effectiveness of the proposed representation-based perturbation method.

We adapt the Spectral Dynamics Embedding Control (SDEC) algorithm from \citet{ren2023stochastic} by adding the $(\xi,\eta)$-perturbation term into the spectral dynamics representation in order to incorporate our robust policy evaluation scheme. The complete algorithm is shown in \Cref{alg:empirical} below. To accommodate the infinite state-action space, we perform SGD for regularized objective (see \Cref{sec:large_state_space}) to solve the optimization, as shown in line \ref{line:pertubation} of \Cref{alg:empirical}. The nominal $Q$-factor $\bomega^{\circ,k}$ is approximated by value iteration in the infinite-horizon setting. The perturbation radii are selected as $R_{\xi} = R_{\eta} = 3$ based on estimated magnitude of the feature and factor vectors.

\begin{algorithm}[h]
  \caption{Representation-Robust Stochastic Dynamics Embedding Control}\label{alg:empirical}
  \begin{algorithmic}[1]
    \State Sample $w_i \sim \mathcal{N}(0, \sigma^{-2} I_n)$ and $b_i \sim \mathsf{Unif}([0, 2\pi])$ i.i.d. to construct $\bphi^{\circ}(s,a)$ and $\bmu^{\circ}(s')$ as in \eqref{eq:D-spectral-feature}.
    \State Initialize policy $\pi^0(\cdot | s) \gets \mathsf{Unif}(\A)$.
    \For{$k = 0, 1, \cdots, K$}
      \State Sample $\{(s_i, a_i, s'_i, a'_i) \}_{i \in [N]}$, where $(s_i, a_i) \sim d^{\pi^k}$, $s'_i = f(s_i, a_i) + \epsilon$, and $a'_i \sim \pi^k(s'_i)$.
      \State Initialize $\bomega^k_0 \gets \bm{0}$. \quad\textcolor{gray}{\textit{// We approximate $\bomega^{\circ,k}$ using value iteration for the infinite-horizon setting.}}
      \For{$t = 0, 1,\dots, T$}
        \State Perform LSVI update: $\bomega^{k}_{t+1} \gets \arg \min_{\bomega} \brac*{ \sum_{i \in [N]} \prn[\big]{ \angl{\bphi^k(s_i, a_i), \bomega} - r(s_i, a_i) - \gamma \angl{ \bphi^k(s'_i, a'_i), \bomega^{k}_{t}} }^2 }$.
      \EndFor
      \State Set $\bomega^{\circ, k} \gets \bomega^{k}_{T+1}$, and perform robust policy evaluation: \label{line:pertubation}
      \begin{equation}
        \bm{\eta}^{k+1}, \bxi^{k+1} \gets \min_{\bm{\eta}, \bxi} \frac{1}{N} \sum_{i \in [N]} \angl[\big]{\bphi^{\circ}(s_i,a_i) + \bm{\eta}, \bomega^{\circ,k} + \bxi} + \lambda_{\xi} \prn[\big]{\norm{\bxi}^2 - R_{\xi}^2} + \lambda_{\eta} \prn[\big]{\norm{\bm{\eta}}^2 - R_{\eta}^2}.
      \end{equation}
      \State Perform feature update: $\bphi^{k+1}(s,a) \gets \bphi^{\circ}(s,a) + \bm{\eta}^{k+1}$, $\bomega^{k+1} \gets \bomega^{\circ,k} + \bxi^{k+1}$.
      \State Use Natural Policy Gradient to update the policy: $\pi^{k+1}_h(a | s) \propto \pi^k_h(a | s) \cdot \exp\prn[\big]{\alpha \angl{\bphi^{k+1}(s,a), \bomega^{k+1}}}$.
    \EndFor
  \end{algorithmic}
\end{algorithm}

The key idea behind the SDEC algorithm is that, assuming i.i.d. Gaussian noise, the transition probability is simply a Gaussian distribution $\P(\cdot | s, a) \mathrel{\overset{\mathrm{d}}{=}} \mathcal{N}\prn[\big]{ f(s, a), \sigma^2 I_n }$, which has a Gaussian kernel representation in terms of $f(s, a)$ and $s'$. Based on Bochner theorem \citep{devinatz1953integral}, we can obtain a finite approximation of the representation by truncating the kernel representations; specifically, we shall take
\begin{subequations}\label{eq:D-spectral-feature}
\begin{align}
  \bphi^{\circ}(s,a) &= \brak[\big]{ 
    \cos(w_1^\top f(s,a)+b_1)
    ~ \cdots ~
    \cos(w_m^\top f(s,a) + b_m)
  }^{\top},\\
  \bmu^{\circ}(s') &= \brak[\big]{
    \cos(w_1^\top s'+b_1)
    ~ \cdots ~
    \cos(w_m^\top s' + b_m)
  }^{\top},
\end{align}
\end{subequations}
where $w_i \sim \mathcal{N}(0,\sigma^{-2}I_d)$  and $b_i \sim \mathsf{Unif}([0, 2\pi])$ are drawn i.i.d. ($i \in [m]$).

\paragraph{Evaluation.} To evaluate whether the output policy displays robustness in its behavior, the policy obtained via \Cref{alg:empirical} (labelled ``robust'') is evaluated with pendulum mass perturbed around the nominal mass $m^{\circ} = 1.0$, namely $m \in \brac{0.2,0.6,1.0,1.4,1.8}$, and compared against the original non-robust SDEC (labelled ``non-robust''). For each perturbed setting, both policies are tested for 50 episodes with randomly sampled initial states, and the discounted cumulative rewards are recorded. The results are plotted in \Cref{fig:A-4-exp_inverted_pen}. It is evident that, when the actual mass of the pendulum is perturbed to be heavier (so that the task becomes harder), the robust SDEC algorithm suffers from less severe performance degradation than non-robust SDEC, and they yield comparable performance when the mass is perturbed to be lighter (so that the task becomes easier). This phenomenon demonstrates that the proposed duple-perturbation robust policy evaluation scheme does help to promote robustness of the output policies, especially for those perturbations that pose significant challenges to the task, and thus justifies the effectiveness of the proposed robustness concept in real-world settings.

\begin{figure}[h]
  \centering
  \includegraphics[width=0.32\linewidth]{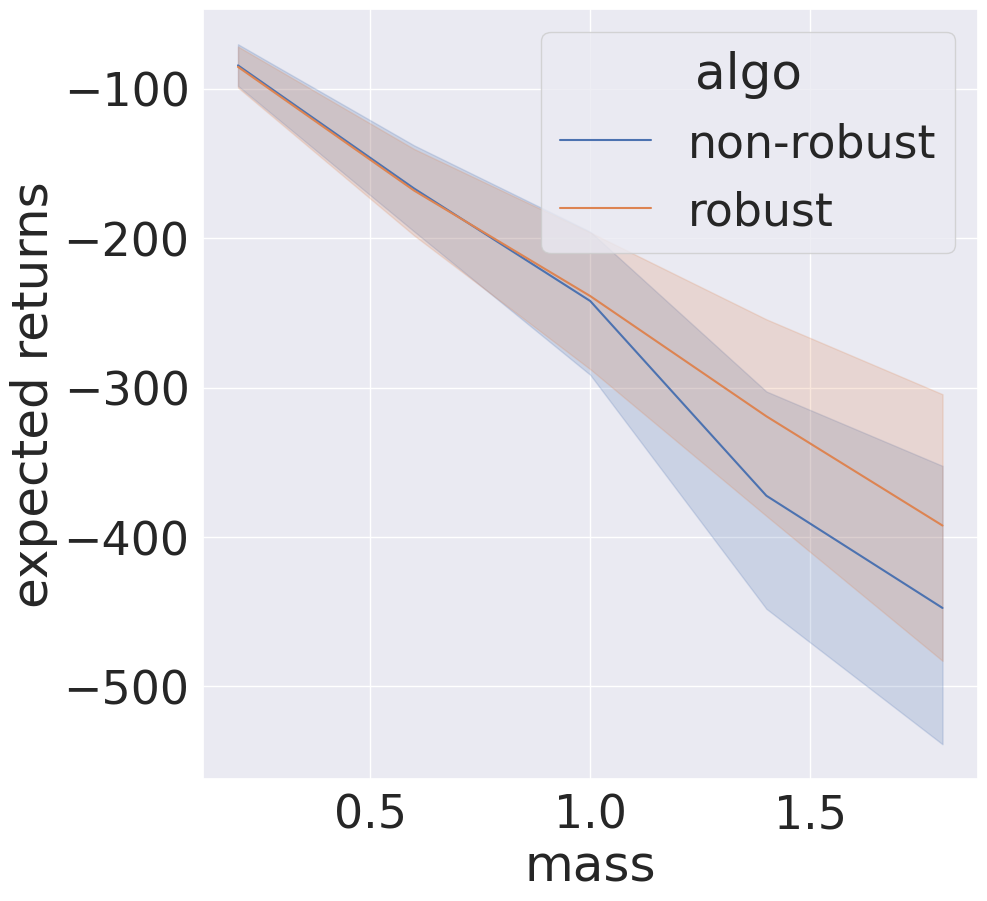}
  \caption{Evaluation results with the perturbed mass of the pendulum.(lines and shaded regions indicate the mean cumulative reward and a 95\% confidence region over 50 evaluation episodes)}
  \label{fig:A-4-exp_inverted_pen}
\end{figure}

\end{document}